\documentclass[12pt]{article}

\usepackage[a4paper,top=3cm,bottom=2cm,left=2cm,right=2cm,marginparwidth=1.25cm]{geometry}

\usepackage{amsthm}
\usepackage{amssymb}

\usepackage{algorithmic}

\usepackage{algorithm}
\usepackage{float}
\usepackage{subfigure}
\usepackage{graphicx}
\usepackage{amsmath}
\usepackage{wrapfig}
\usepackage{verbatim}
\usepackage{hhline}

\usepackage{mathabx}

\usepackage[title]{appendix}

\usepackage{xcite}

\usepackage{xr}
\makeatletter

\graphicspath{{../pdf/}{../jpeg/}{/graphs/}}    
\DeclareGraphicsExtensions{.pdf,.jpeg,.png,.jpg}
  \newcommand{\R}{\ensuremath{\mathbb{R}}}

  \newcommand{\Nc}{\mathcal{N}}

\newcommand{\Pd}[3]{\ifthenelse{\equal{#3}{1}}{\frac{\partial #1}{\partial #2}}{\frac{\partial^{#3} #1}{\partial #2^{#3}}}}

\newtheorem{Lemma}{Theorem}

\newtheorem{lemma}[Lemma]{Lemma}

\usepackage{lineno}

\newcommand{\network}{\V{h}}
\newcommand{\observation}{\V{f}}

\newcommand{\encoder}{{\V{\rho}}}

\newcommand{\decoder}{{\V{\gamma}}}
\newcommand {\defeq}{\triangleq}
\newcommand{\dmaps}{\V{\psi}}
\newcommand{\admaps}{\V{\phi}}

\newcommand{\wifiTransmissionDistance}{600}

\newcommand{\wifiAccessPointCountNumeric}{17}
\newcommand{\numSensorReadings}{4,000}
\newcommand{\sensorArrayRadiusPixels}{0.5}

\newcommand{\figref}[1]{Fig. \ref{#1}}
\newcommand{\vast}{\bBigg@{4}}
\newcommand{\Vast}{\bBigg@{5}}

\newcommand {\V}[1] {{\mbox{\boldmath $#1$}}}
\title{LOCA: LOcal Conformal Autoencoder \\
for standardized data coordinates}

\author{Erez Peterfreund $^{1 \ast }$ \and Ofir Lindenbaum $^{2 \ast }$ \and Felix Dietrich $^{3}$ \and Tom Bertalan$^3$ \and Matan Gavish$^1$ \and Ioannis G. Kevrekidis$^3$ \and Ronald R. Coifman$^{2, \dagger}$\\
\\
\normalsize{$^{1}$Hebrew University of Jerusalem;} 
\normalsize{$^{2}$Yale University;}\\
\normalsize{$^{3}$Johns Hopkins University;}\\
\normalsize{$^\dagger$Corresponding author. E-mail: ronald-coifman@yale.edu}\\
\normalsize{$^\ast$ These authors contributed equally.}
}

\date{}

\begin{document}

\maketitle
\begin{abstract}
 We propose a deep-learning based method for obtaining standardized data coordinates from scientific measurements.
 Data observations are modeled as samples from an unknown, non-linear deformation of an underlying Riemannian manifold, which is parametrized by a few normalized latent variables. By leveraging a repeated measurement sampling strategy, we present a method for learning an embedding in $\mathbb{R}^d$ that is isometric to the latent variables of the manifold. 
 These data coordinates, being invariant under smooth changes of variables, enable matching between different instrumental observations of the same phenomenon. Our embedding is obtained using a LOcal Conformal Autoencoder (LOCA), an algorithm that constructs an embedding to rectify deformations by using a local z-scoring procedure while preserving relevant geometric information. We demonstrate the isometric embedding properties of LOCA on various model settings and observe that it exhibits promising interpolation and extrapolation capabilities. Finally, we apply LOCA to single-site Wi-Fi localization data, and to $3$-dimensional curved surface estimation based on a $2$-dimensional projection.
\end{abstract}

\section{Introduction}
Reliable, standardized tools for analyzing complex measurements are crucial for science in the data era. Experimental data often consist of multivariate observations of a physical object that can be represented as an unknown Riemannian manifold. A key challenge in data analysis involves converting the observations into a meaningful and, hopefully, intrinsic parametrization of this manifold. For example, in astrophysics, one is interested in a representation that is coherent with the material composition of stars based on measurable high dimensional spectroscopic data \cite{sdss1,sdss2}. This type of challenge has typically been studied under the broader umbrella of dimensionality reduction and manifold learning, where numerous algorithmic solutions have been proposed \cite{PCA,MDS,Isomap,Leig,Dmaps,LLE,Auto-e,tsne,umap,perraul2013non}. These methods rely on statistical or geometrical assumptions and aim to reduce the dimension while preserving different affinities of the observed high dimensional data.  

In this paper, we focus on data obtained from several observation modalities measuring a complex system. These observations are assumed to lie on a path-connected manifold which is parameterized by a small number of latent variables. We assume that the measurements are obtained via an unknown nonlinear measurement function observing the inaccessible manifold. The task is then to invert the unknown measurement function, so as to find a representation that provides a standardized parametrization of the manifold. In general, this form of blind inverse problem may not be feasible. Fortunately, in many cases, one can exploit a localized measurement strategy, suggested in \cite{nonlinearICA}, to extract an embedding into internally standardized (z-scored) latent variables.

Toward a useful formulation of our problem we note that, in numerous real-world scenarios, it is possible to capture data using a localized ``burst'' sampling strategy \cite{rabin2012heterogeneous,ronen_localization,kushnir2012anisotropic,ronen_eeg,crack_loc,ronen_intrinsicIsomap,crack_loc,Amit2,shevchik20193d,talmon2012differential}. To motivate this type of ``burst'' sampling, we describe a toy experiment (see \figref{fig:motivation}). Consider the task of recovering the geometry of a curved $2$-dimensional homogeneous surface in three dimensions using a laser beam, which heats the surface locally at several positions. Here, a ``burst'' is realized through the brief local isotropic propagation of heat around each laser impact location (each ``data point"), which can be visualized as a local ellipse by a thermal camera. Now, the task is to recover the curved geometry of the surface in three dimensions using the collection of observed local $2$-dimensional ellipses. 

More generally, our strategy is realized by measuring such brief ``bursts'', which are modeled as local isotropic perturbations added to each state in the inaccessible latent manifold. The ``bursts'' provide information on the local variability in a neighborhood of each data point. Thus, they can be used to estimate the Jacobian (modulo an orthogonal transformation, as we will discuss) of the unknown measurement function. The authors in \cite{nonlinearICA} use such ``bursts'' and suggest a scheme to recover a representation that is invariant to the unknown transformation. Specifically, they use a local Mahalanobis metric, combined with eigenvectors of an Anisotropic Laplacian, in order to extract a desired embedding.

Solutions such as \cite{nonlinearICA} and extensions such as \cite{talmon2013empirical,Amit2,carmelin,felix,Multiview} can be used. There remain, however, several challenges: (i) they require a dense sampling of the deformed manifold. (ii) they deform the representation due to inherent boundary effects, and (iii) they do not extend easily to unseen samples. To overcome these limitations, we introduce the concept of a {\em LOcal Conformal Autoencoder} (LOCA), a deep-learning based algorithm specifically suited to ``burst'' measurements. LOCA is realized using an encoder-decoder pair, where the encoder attempts to find a mapping such that each ``burst'' is locally whitened (z-scored). By forcing reconstruction, the decoder ensures that geometric information has not been lost. We have found LOCA to be scalable, as well as easy to implement and parallelize using the existing deep learning open-source codebase. We provide empirical evidence that the LOCA embedding is approximately isometric to the latent manifold and extrapolates reliably to unseen samples. We discuss a scheme to automatically tune the minimal embedding dimension of LOCA, and demonstrate its precision in two real data problems. 

The contributions in this paper are as follows. (i) We show that the localized sampling strategy (our ``bursts'' at a given scale) generates a consistent Riemannian structure; under certain conditions, it allows inverting the unknown measurement function (modulo a shift and orthogonal transformation). 
(ii) We present a two-step optimization scheme for learning a parametric mapping, which is approximately an isometry of the latent manifold. 
(iii) We demonstrate the isometric properties of the extracted encoder on several examples.
(iv) We verify empirically that the extracted neural network has good interpolation and extrapolation properties.

\begin{figure}
    \centering
    \includegraphics[width=0.6\textwidth]{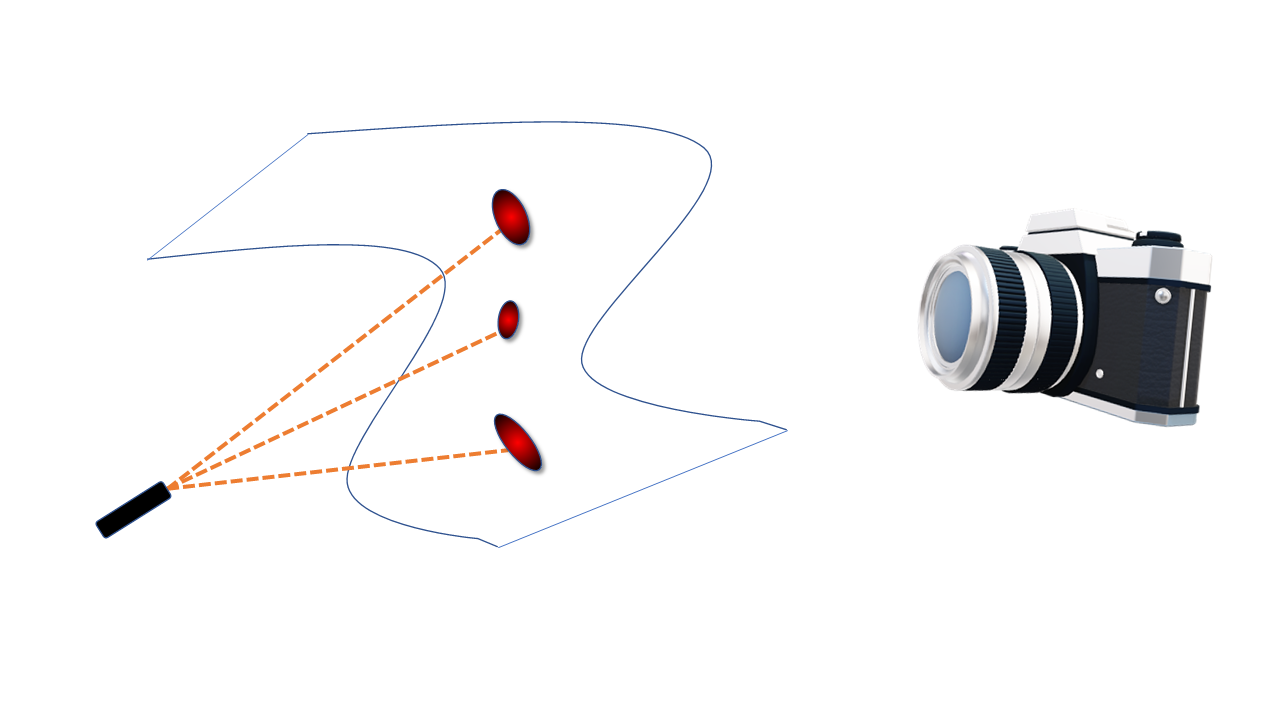}
    \caption{A motivating example for LOCA- learning the stereographic shape of a surface. We consider a laser beam used to locally heat the surface at several positions. A thermal camera measures the brief isotropic propagation of heat around each location.  By scanning a thermal image, we can identify the neighborhoods of each position, which we define as our ``bursts". LOCA uses these ``bursts" to invert the projection and recover a scaled version of the curved surface.}
    \label{fig:motivation}
\end{figure}

%Our method locally corrects each Burst 

\section{Problem Settings}
\label{sec:problem settings}

\textbf{The burst measurement strategy.} 

Consider first, for simplicity, the case where the {\em latent} domain for our system of interest is a path-connected domain in a Euclidean space $\mathcal{X}=\mathbb{R}^d$. We call $\mathcal{X}$ the {\em latent} space. Observations of the system consist of samples captured by a  measurement device given as a non-linear function $f:\mathcal{X}\rightarrow \mathcal{Y}$, where $\mathcal{Y}$ is the ambient, or ``measurement'' space. Even if $\observation$ is invertible, it is generally not feasible to identify $\observation^{-1}$ without access to $\mathcal{X}$. Here, we assume that (a) $\observation$ is smooth and injective; and that (b) multiple, slightly perturbed versions of the physical system point in $\mathcal{X}$ give rise to multiple (slightly perturbed) measurements in $\mathcal{Y}$. In this notation, by exploiting a specific type of local perturbation, we develop a method to recover a standardized version $\mathcal{X}$ from $\mathcal{Y}$  (up to an approximately isometric transformation which, for Euclidean spaces would be a rigid transformation).

Consider $N$ data points (``burst" centers), denoted $\V{x}_1,\ldots,\V{x}_N\in \mathbb{R}^d$ in the latent space. Assume that all these points lie on a path-connected, $d$-dimensional sub-domain of $\mathcal{X}$; we will later discuss the restriction to Riemannian manifolds with a smaller dimension than the full space. Importantly, we do not have direct access to these states. The states are pushed forward to the ambient space via the unknown deformation which defines $\V{y}_1,\ldots \V{y}_N\in \mathbb{R}^D$. We do not only observe the states $\V{y}_1,\ldots,\V{y}_N$; we rather assume a specific perturbed sampling strategy. For each $\V{x}_i$, $i=1,...,N$, we assume that the perturbed observed state is given as the random variable  
\begin{eqnarray}
\label{eq:ambient_clouds}
\V{Y}_{i}= \observation(\V{x}_i+\V{Z}_i)\in \mathbb{R}^D \quad i=1,\ldots,N,
\end{eqnarray}
where $\V{Z}_1,\ldots,\V{Z}_n$ are i.i.d distributed by $\mathcal{N}_d(\V{0},\sigma^2 \V{I}_d)$. Our sampling strategy relies on measuring $M$ perturbed realizations of $\V{Y}_i$, which we denote as $\V{y}^{(j)}_i,j=1,...,M$. We assume that $\sigma \ll 1$, or alternatively, that $\sigma$ is sufficiently small such that the differential of $\observation$ practically does not change in a ball of radius $\sigma$ around any point. Such sufficiently small $\sigma$ allows us to capture the local neighborhoods of the states {\em at this measurement scale} on the latent manifold. Note that alternative isotropic distributions, which satisfy this condition, could be used to model $\V{Z}$. 

Let us explore the implications of this localized sampling strategy for learning a representation that is consistent with $\mathcal{X}$. Specifically, our goal is to construct an embedding $\encoder$ that maps the observations $\V{y}_i$, so that the image of $\rho\circ \observation$ is isometric to $\mathcal{X}$ when $\sigma$ is known. In our Euclidean setting, such an isometric embedding should satisfy 
\begin{eqnarray}
\label{eq:preserve_dists}
\| \encoder(\V{y}_i)-\encoder(\V{y}_j) \|_2 = \| \V{x}_i-\V{x}_j \|_2, \text{ for any } i,j=1,\ldots,N.
\end{eqnarray}
We note that if $\sigma$ is not known, we will relax \eqref{eq:preserve_dists} by allowing a global scaling of the embedding. This means that we are only looking for a representation that preserves the pairwise Euclidean distances between the latent samples, rather than obtaining their actual values. More specifically, a $\encoder$ that satisfies \eqref{eq:preserve_dists} is not unique, and is defined up to an isometric transformation of the data. We refer to representations which satisfy Eq. \ref{eq:preserve_dists} up to errors smaller than $\sigma$ as ``isometries".

\section{Related work}
The problem of finding an isometric embedding was also studied in \cite{mcqueen2016nearly}. The paper proposes an algorithm to embed a manifold with dimension $d$ into a space of dimension $s \geq d$. The method in \cite{mcqueen2016nearly} is built upon \cite{perraul2013non} and uses a discrete version of the Laplace-Beltrami operator, as in \cite{Dmaps}, to estimate metric of the desired embedding. To force the embedding to be isometric to the observed samples $\mathcal{Y}$, the authors propose a loss term to quantify the deviation between the push-forward metric from the observed space, $\mathcal{Y}$, to the embedding space compared with the restricted Euclidean metric. The embedding is refined by applying gradient descent to the proposed loss. The approach successively approximates a Nash embedding with respect to the observed space, $\mathcal{Y}$, for which it is required that the manifold be densely sampled at all scales.

In this work, we use ``bursts'' to learn an embedding that corrects the deformation $\observation$ and isometrically represent the inaccessible manifold $\mathcal{X}$. The idea of using ``bursts'', or data neighborhoods, to learn a standardized reparametrization of data was first suggested in \cite{nonlinearICA}. The authors assume the data is obtained via some unknown nonlinear transformation of latent independent variables. Locally, the distortion caused by the transformation is corrected by inverting a Jacobian that is estimated from the covariances of the observed ``bursts''. This allows the authors to define a {\em local} Mahalanobis metric (which is affine invariant). Then, this metric is used to construct an anisotropic "intrinsic Laplacian" operator. Finally, the eigenvectors of this Laplacian provide the independent components and are used as a ``canonical" embedding. This framework was extended in several studies such as \cite{carmelin,ronen_localization,talmon2013empirical,Multiview}.

The work of these authors can be improved in three directions. First, they require inverting a covariance matrix in the ambient space. Second, they suffer from deformation on boundaries. Third, they typically do not provide an embedding function that can be naturally extended over the entire data domain and beyond; instead, they provide a specific mapping for the existing training samples. This last direction means that to embed test data, methods such as \cite{coifman2006geometric,rabin2012heterogeneous} could be employed. The mapping approximations based on these methods are limited and cannot extend further than a small neighborhood around each training point. Furthermore, even though the provided embedding is unique, it is not isometric to the latent variables. We present a method that alleviates these shortcomings, and empirically demonstrate that it extracts a canonical representation that is isometric to the latent variables.

Perhaps the most related work to this study was recently presented by \cite{ronen_intrinsicIsomap}. The authors consider ``bursts'' to develop a method for finding an embedding that is isometric to the latent variables. They build upon Isomap \cite{tenenbaum2000global}, and use two neural networks to refine the Isomap based embedding. The first neural network is used in order to obtain a continuous model for estimating the covariance $C(\V{Y}_i)$. The covariances are used for calculating local Mahalanobis based distances, which are fed into Isomap to obtain an initial embedding. Next, they train an additional neural network to correct the Isomap embedding so that the Euclidean distances will approximate the local Mahalanobis distances. In this paper, we take a different and, we believe, more systematic/general approach by presenting a simple encoder-decoder pair (see Fig. \ref{fig:loca illustration}) that is directly applicable to samples in the observed high dimensional space. Specifically, our approach provides a parametric mapping that allows us to extend the embedding to new unseen samples naturally. Furthermore, we learn the inverse mapping, which could be used to generate new samples by interpolating in the latent space.

\begin{figure}[t]
    \centering
    \includegraphics[width=0.6\textwidth]{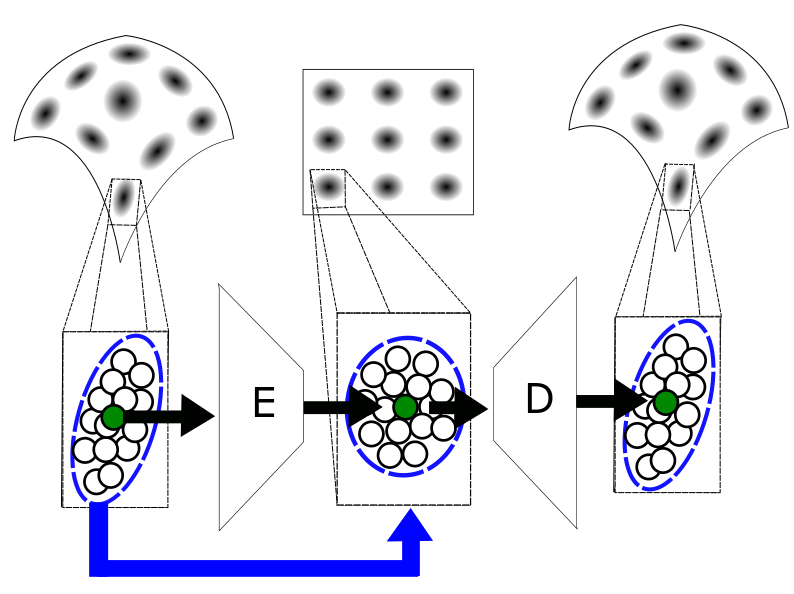}
    \caption{An illustration of the LOcal Conformal Autoencoder (LOCA). ``E'' stands for the encoder $\encoder$, and ``D'' for the decoder $\decoder$. The autoencoder receives a set of points {\em along with corresponding ``neighborhoods''}; each neighborhood is depicted as a dark oval point cloud (see the top row in the figure). On the bottom row, we zoom in onto a single ``anchor'' point $\V{y}_i$ (green) along with its corresponding neighborhood $\V{Y}_i$ (bounded by a blue ellipsoid). The encoder attempts to ``whiten'' each neighborhood in the embedding space, while the decoder tries to reconstruct the input.}
    \label{fig:loca illustration}
\end{figure}

\section{Deriving an Alternative Isometry Objective}
Without access to samples from $\mathcal{X}$ the objective described in  \eqref{eq:preserve_dists} does not provide any information for extracting $\encoder$. Here we reformulate this objective by utilizing the special stochastic sampling scheme presented in Section \ref{sec:problem settings} and relate it to the differential equation for the embedding described in Lemma~\ref{lem:cov_jacobian}. We start by plugging the unknown measurement function into  \eqref{eq:preserve_dists}; then, we can approximate its left hand side using a first order Taylor expansion
\begin{eqnarray*}
\| \encoder(\V{y}_i)-\encoder(\V{y}_j) \|_2 &=& 
\| \encoder\circ \observation (\V{x}_i)-\encoder\circ \observation(\V{x}_j) \|_2\\
& \approx &
\| \V{J}_{\encoder\circ \observation} (\V{x}_i) \left(\V{x}_j- \V{x}_i \right) \|_2.
\end{eqnarray*}
Hence, by neglecting higher order terms, we can define the following objective 
\begin{eqnarray}
\label{eq:derive alternative objective 1}
\V{J}_{\encoder\circ \observation} (\V{x}_i) ^T
\V{J}_ {\encoder\circ \observation} (\V{x}_i) = \V{I}_d,\text{ for } i=1,
\ldots,N,
\end{eqnarray}
which allows us to evaluate the isometric property of  $\encoder$.

Now we want to relate the Jacobian in Eq. \ref{eq:derive alternative objective 1} to measurable properties of the observations $\V{Y}_1,\ldots,\V{Y}_N$. Specifically, we can rely on the following Lemma to approximate the derivatives of the unknown function $\observation$ at each point $\V{x}_1,\ldots,\V{x}_N$. The Lemma is proved in the Supporting Information (\ref{sec:proof of lemma}):
\begin{lemma}
\label{lem:cov_jacobian}
Let $\V{g}:\mathcal{X} \rightarrow \mathcal{Z}$ be a function, where $\mathcal{X}=\mathbb{R}^d$ and $\mathcal{Z}=\mathbb{R}^D$. 
Let $\V{x}\in \mathcal{X}$ and $\sigma\in \mathbb{R}_{+}$. Define a random variable $\V{X}\sim \mathcal{N}(\V{x},\sigma^2 \V{I}_d)$. 
If the function satisfies $\V{g}\in \mathcal{C}^3$ and is injective, there exist a $\sigma \in \mathbb{R}_{+}$ such that the covariance of the transformed random variable $\V{Z}=\V{g}(\V{X})$ is related to the Jacobian of $\V{g}$ at $\V{x}$ via
\begin{eqnarray*}
\V{J_g} (\V{x}) \V{J_g} (\V{x})^{T} &=& \frac{1}{\sigma^2}\V{C}(\V{Z})   + O(\sigma^2).
\end{eqnarray*}  
Moreover,
\begin{eqnarray*}
\frac{1}{\sigma^2} \V{C}(\V{Z}) \underset{\sigma\rightarrow 0}{\longrightarrow} \V{J_{g}} (\V{x}_i) \V{J_{g}} (\V{x})^T .
\end{eqnarray*}
\end{lemma}

By setting $\V{g}\equiv\encoder \circ \observation$, this Lemma provides a relation between the Jacobian of $\encoder\circ \observation$ and the covariance in the embedding space. Specifically, this translates to a system of differential equations for the Jacobian of an isometric (Nash) embedding 
\begin{eqnarray}
\label{eq:Jacobian covariance configuration}
\V{J}_{\encoder\circ \observation } (\V{x}_i)
\V{J}_{\encoder\circ \observation } (\V{x}_i)^{T} &=&
\frac{1}{\sigma^2} \V{C}\left(\encoder( \V{Y}_i)\right) +O(\sigma^2).
\end{eqnarray}
When $D=d$, we can tie the approximation of objective \eqref{eq:preserve_dists} with \eqref{eq:Jacobian covariance configuration} by
\begin{eqnarray}
\label{eq:whitening objective}
\frac{1}{\sigma^2}\V{C}\left(\encoder( \V{Y}_i)\right)&=& \V{I}, \text{ for any } i=1,\ldots,N.
\end{eqnarray} 
Thus we can evaluate the embedding function at each point without gaining access to the latent states of the system.

\label{sec:algo}
\begin{algorithm}[t]
    \caption{LOCA: LOcal Conformal Autoencoder} \textbf{Input}: Observed clouds $\V{Y}_i$,i=1,...,N. \\
{\textbf{Output}}: $\V{\theta}_e$ and $\V{\theta}_d$ - the weights of the encoder $\encoder$ and decoder $\decoder$ neural network.
    \begin{algorithmic}[1] 
    \FOR{$t=1,...,T$}
        \STATE Compute the whitening loss 
           $$L_{white} = \frac{1}{N} \sum_{i=1}^{N} \big\| \frac{1}{\sigma^2}\widehat{\V C}\left(\encoder(\V{Y}_i)\right) - \V{I}_d \big\|_F^2$$
           
           \STATE Update $\V{\theta}_e := \V{\theta}_e - \eta \nabla_{\V{\theta}_e} {L}_{white}$
            \STATE Compute the reconstruction loss 
         $$L_{recon} = \frac{1}{N\cdot M} \sum_{i,m=1}^{N,M} \big\| \V{y}_i^{(m)} - \decoder\left( \encoder\left(\V{y}_i^{(m)}\right) \right)\big\|_2^2$$
            
            \STATE Update $\V{\theta}_e := \V{\theta}_e - \eta \nabla_{\V{\theta}_e} {L}_{recon}$ and\newline  \text{\qquad\qquad\,\,\,\,} $\V{\theta}_d := \V{\theta}_d - \eta \nabla_{\V{\theta}_d} {L}_{recon}$
    \ENDFOR
    \end{algorithmic}
    \label{alg:LOCA}
\end{algorithm}
\section{LOcal Conformal Autoencoder}

We now introduce the Local COnformal Autoencoder (LOCA), with training Algorithm \ref{alg:LOCA}. Our method is based on optimizing two loss terms; the first is defined based on \eqref{eq:whitening objective} using what we refer as a ``whitening'' loss
\begin{eqnarray}
\label{eq:white_loss}
L_{white} (\encoder) = \frac{1}{N} \sum_{i=1}^N \Bigg\| \frac{1}{\sigma^2}\widehat{\V C}\left(\encoder(\V{Y}_i)\right) - \V{I}_d \Bigg\|_F^2,
\end{eqnarray} where $\encoder$ is an embedding function, and $\widehat{\V{C}}\left(\encoder( \V{Y}_i)\right)$ is the empirical covariance over a set of $M$ realizations $\encoder\left(\V{y}_i^{(1)}\right),\ldots, \encoder\left(\V{y}_i^{(M)}\right)$, where  $\V{y}_i^{(1)},\ldots, \V{y}_i^{(M)}$ are realizations of the random variable $\V{Y}_i$.

\begin{figure*}[!htb]
\centering
\subfigure[ ]{\label{fig:x_mush} \includegraphics[width=0.32\textwidth] {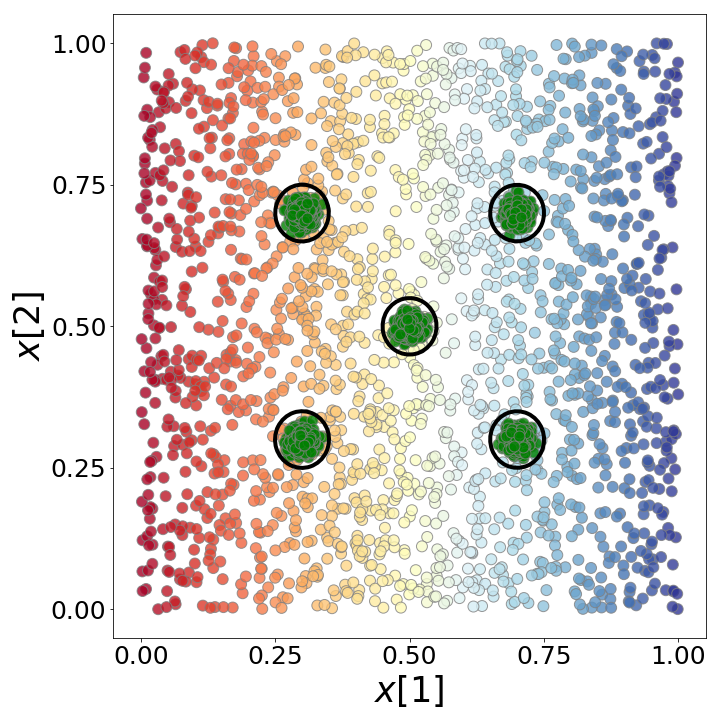}}
\subfigure[ ]{\label{fig:y_mush} \includegraphics[width=0.32\textwidth] {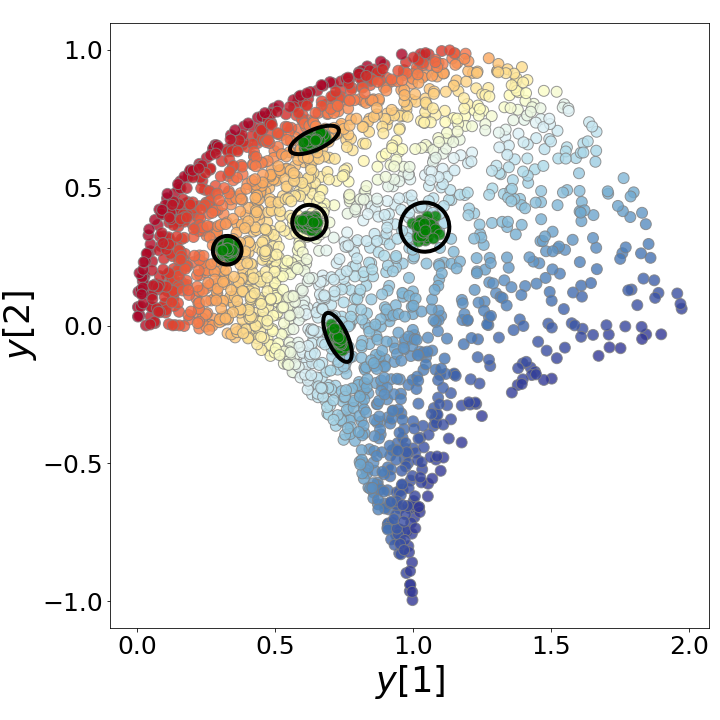}}
\subfigure[]{\label{fig:dists_mush}\includegraphics[width=0.32\textwidth]{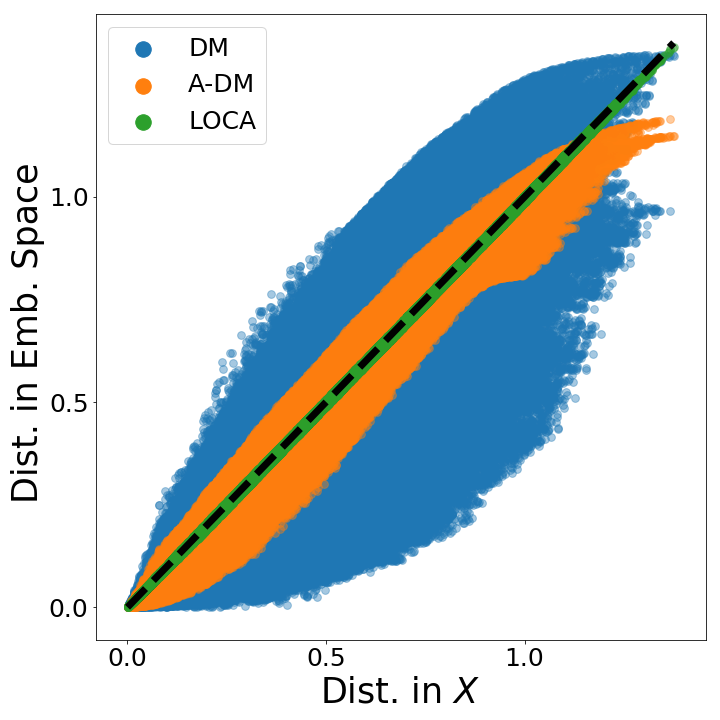}}
\caption{ Evaluating the isometric quality of the embedding (the setting is detailed in \ref{sec:experimets_isometry}).
The points $\V{x}_i,i=1,...,N=2000$ in the latent space of the system of interest \subref{fig:x_mush} are  \subref{fig:y_mush} pushed forward to the measurements $\V{y}_i=1,...,N$ by applying the nonlinear transformation $\V{f}_1$(described in \eqref{eq:experiment_tranform_mush}). The color in both figures corresponds to the values of $x[1]$ of the data. We observe ``bursts'' around each sample (based on the burst model described in Section \ref{sec:problem settings}). To illustrate this ``burst'' sampling scheme, we overlay the points with additional green samples generated by bursts at $5$ different positions. 
\subref{fig:dists_mush} Euclidean distances between pairs of points in the latent space plotted versus the corresponding Euclidean distance in the embedding space. The corresponding distances for DM and for A-DM are also shown in color, scaled with a factor which minimizes the stress (defined in \eqref{eq:stress}).
}
\label{fig:mushroom_xy}
\end{figure*}
As $\observation$ is invertible on its domain, an embedding function $\encoder$ that approximates $\observation^{-1}$ should be invertible as well. The invertibility of $\encoder$ means that there exists an inverse mapping $\V{\decoder}: \mathbb{R}^d  \longrightarrow \mathcal{Y}$, such that  $\V{y}_i = \V{\decoder}(\V{\encoder}(\V{y}_i))$ for any $i\in [N]$.
This additional objective helps remove undesired ambiguities (which may occur for insufficient sampling). By imposing an invertibility property on $\encoder$, we effectively regularize the solution of $\encoder$ away from  noninvertible functions. To impose invertibility, we define our second loss term, referred to as "reconstruction" loss:
\begin{eqnarray}
\label{eq:recon_loss}
L_{recon} (\encoder, \decoder) = \frac{1}{N\cdot M} \sum_{i,m} \Bigg\| \V{y}_i^{(m)} - \decoder\left( \encoder\left(\V{y}_i^{(m)}\right) \right)\Bigg\|_2^2.
\end{eqnarray}

We suggest finding an isometric embedding based on an autoencoder, where $\encoder$ will be defined as the encoder and $\decoder$ as the decoder. 
We construct solutions to~\eqref{eq:white_loss} and \eqref{eq:recon_loss} with a neural network ansatz $\encoder= {\network}_e^{(L)}$ and $\decoder= {\network}_d^{(L)}$  consisting of $L$ layers each, such that 

\begin{eqnarray*}
{\network}_e^{(\ell)}(y)&=&\V{\sigma}_e\left(\V{W}_e^{(\ell-1)} {\network}_e^{(\ell-1)}(y)+\V{b}^{(\ell-1)}\right),\  \ell=1,\dots,L,
\\
{\network}_d^{(\ell)}(z)&=&\V{\sigma}_d\left(\V{W}_d^{(\ell-1)} {\network}_d^{(\ell-1)}(z)+\V{b}_d^{(\ell-1)}\right),\  \ell=1,\dots,L,
\end{eqnarray*}
where ${\network}_e^{(0)}(y)=y$ and ${\network}_d^{(0)}(z)=z$.
\noindent Here, $\V{W}_e^{\ell}$, $\V{b}_e^{\ell}$ and $\V{W}_d^{\ell}$, $\V{b}_d^{\ell}$ are the weights and biases at layer $\ell$ of the encoder and decoder, respectively. The functions $\V{\sigma}_e,\V{\sigma}_d$ are nonlinear activations applied individually to each input coordinate. As the activation function can have a limited image, we recommend removing the non linear activation for $\ell=L$. 

We propose to find $\encoder$ and $\decoder$  by alternating between a stochastic gradient descent on \eqref{eq:white_loss} and \eqref{eq:recon_loss}. It is important to note that the main objective that we are trying to optimize is based on \eqref{eq:white_loss}, therefore  \eqref{eq:recon_loss} can be viewed as a regularization term. A pseudo-code of this procedure appears in  Algorithm \ref{alg:LOCA}. To prevent over-fitting, we propose an early stopping procedure \cite{li2019gradient,basri2019convergence} by evaluating the loss terms on a validation set. In section \ref{sec:prop}, we demonstrate different properties of the proposed LOCA algorithm using various geometric example manifolds.

Note that functions that perfectly satisfy our objectives are not unique, i.e. for any solution $\V{\rho}$ we can define an equivalent solution $\overline{\V{\rho}}$ that will attain the same loss. Specifically, we can define it by $\overline{\V{\rho}}(\V{y}) = \V{U}\V{y}+ \V{c}$ for any $\V{y}\in \mathcal{Y}$, where $\V{U}\in O(d)$ and $\V{c}\in \mathbb{R}^d$.

\textbf{To summarize:} (i) We collect distorted neighborhoods of a fixed size ($\sigma$) around data points of the system of interest. (ii) We embed/encode the data in a low dimensional Euclidean space so that these neighborhoods are standardized or z-scored. (iii) The embedding is decoded back to the original measurements, to regularize the encoder. In Section \ref{sec:prop}, we demonstrate that (iv) The encoder is invariant to the measurement modality (up to errors of $O(\sigma^2)$, and modulo an orthogonal transformation and shift). (v) The parametric form of the embedding enables reliable interpolation and extrapolation.

\begin{figure*}[t]
\centering
\subfigure[ ]{\label{fig:oos_xFrame} \includegraphics[width=0.23\textwidth] {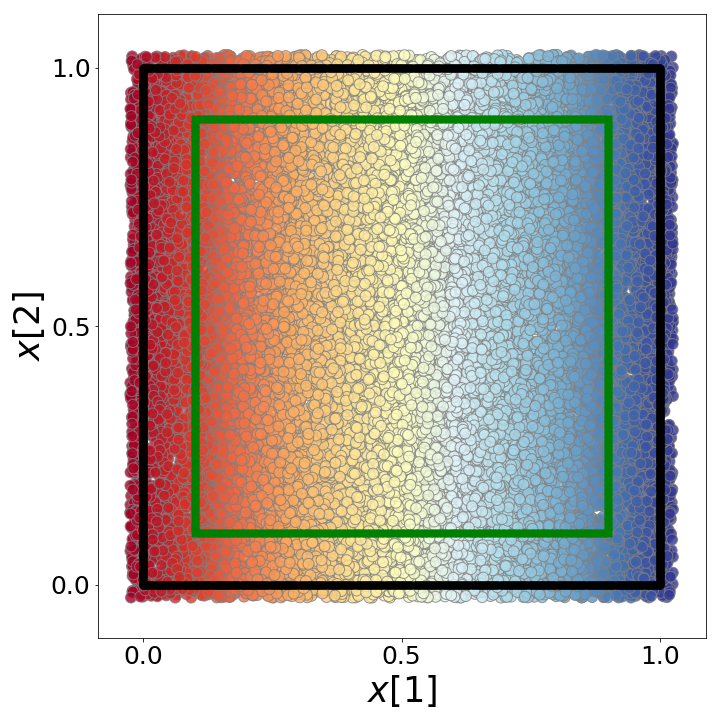}}
\subfigure[ ]{\label{fig:oos_yFrame} \includegraphics[width=0.23\textwidth] {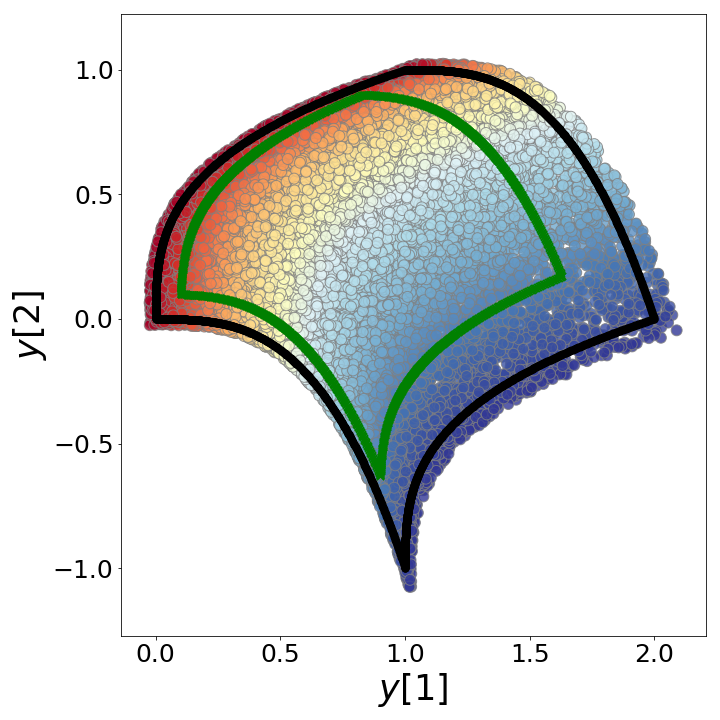}}
\subfigure[ ]{\label{fig:oos_Frame_loca} \includegraphics[width=0.23\textwidth] {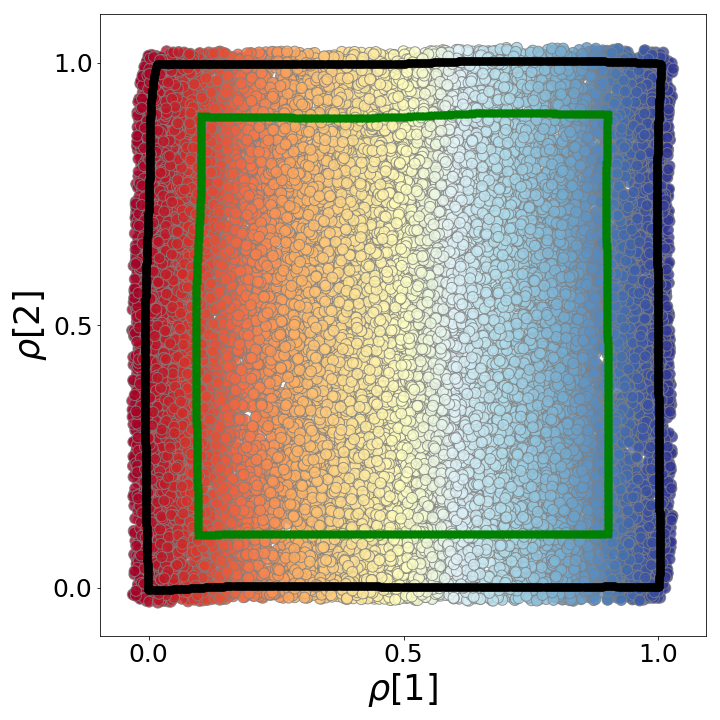}}
\subfigure[ ]{\label{fig:oos_Frame_loca_dists} \includegraphics[width=0.23\textwidth] {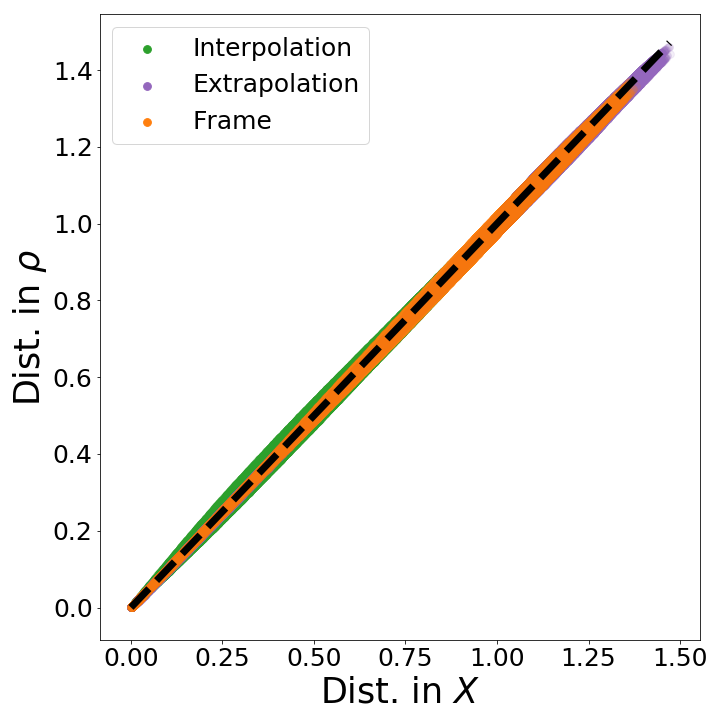}}
\caption{Evaluating out-of-sample performance of the encoder, detailed in Section \ref{sec:out_of_sample}.  \subref{fig:oos_xFrame} The latent space of interest $\V{X}$; our training region is bounded here by the black and green frames. The interpolation region lies within the green frame, while the extrapolation regime lies outside the black frame. \subref{fig:oos_yFrame} The observed space $\V{Y}$ with corresponding regions of interest. \subref{fig:oos_Frame_loca} The calibrated embedding $\V{\rho}$ (using an orthogonal transformation and a shift) with corresponding regions of interest. The color in these figures corresponds to the values of $x[1]$. Here, we calibrated the embedding merely for visualisation purposes. \subref{fig:oos_Frame_loca_dists} The Euclidean distances between pairs of points in the latent space versus the corresponding Euclidean distance in the embedding space.}
\label{fig:mushroom_oos}
\end{figure*}

\section{Properties of LOCA}
\label{sec:prop}

In this section, we evaluate the properties of the proposed embedding $\encoder$ by generating various synthetic datasets. We compare the extracted embedding provided by LOCA (described in Section \ref{sec:algo}) with the embeddings of alternative methods such as Diffusion Maps  (DM) \cite{Dmaps} and Anisotropic Diffusion Maps  (A-DM) \cite{nonlinearICA} denoted as $\dmaps$ and $\admaps$ respectively. Note that the Diffusion Maps algorithm does not use the ``burst'' data, while the Anisotropic Diffusion Maps uses it in order to construct a local Mahalanobis metric. We present the details of the exact implementation for each of the methods in the supplementary information (see Section \ref{sec:dm_adm_description}).

\subsection{LOCA creates an isometric embedding}
\label{sec:experimets_isometry}

We first evaluate the isometric quality of the proposed embedding $\encoder$ with respect to the true inaccessible structure of $\mathcal{X}$. Here, we follow the setting in \cite{nonlinearICA}, where the intrinsic latent coordinates are independent, specifically distributed by $U[0,1]^2$. Based on this distribution, we sample $N=2000$ anchor points $\V{x}_i$ along with ``bursts'' $\V{X}_i$. Each ``burst'' consists of $M=200$ points sampled independently from $\Nc_2(\V{x}_i, \sigma^2 \V{I}_2) $, where $\sigma=0.01$.

We now define the non linear transformation, $f_1:\R^2\rightarrow\R^2$  (as in \cite{nonlinearICA}) to the ambient space by
\begin{eqnarray}
\label{eq:experiment_tranform_mush}
\V{f}_1(\V{x}) &=&  \left( \begin{array}{c}  x[0] +  x[1] ^3\\ 
-x[0]  +   x[1] ^3
\end{array} \right),
\end{eqnarray}
for any $\V{x}\equiv (x[0],x[1])^{\top} \in \R^2$. In Fig.  \ref{fig:mushroom_xy}, we present measurements from $\mathcal{X}$ with the corresponding measurements of $\mathcal{Y}$, where $\mathcal{Y}=f(\mathcal{X})$. To illustrate the local deformation caused by $\V{f}_1$, we overlay the samples with clouds around $5$ different positions (see green dots). Next, we apply LOCA (described in Algorithm \ref{alg:LOCA}) to compute an embedding $\encoder$ that satisfies \eqref{eq:white_loss} and \eqref{eq:recon_loss}. We evaluate the isometric quality of LOCA by comparing the pairwise Euclidean distances in the embedding space $\encoder$ to the Euclidean distances in the latent space $\mathcal{X}$. For comparison, we apply DM and A-DM (that also uses the ``bursts'') to $\mathcal{Y}$ and plot the pairwise Euclidean distances in the embedding vs. the corresponding Euclidean distances in the latent space. Here, we evaluate isometry up to a scaling, as DM and A-DM use eigenvectors (that are typically normalized). The scaling is optimized to minimize the stress defined by
\begin{eqnarray}
\label{eq:stress}
\text{Stress}(\V{g})  &=& \frac{1}{N} \sum_{i,j=1}^N \left(D_x (\V{x}_i, \V{x}_j) - D_{\V{g}} (\V{y}_i, \V{y}_j)\right)^2
\end{eqnarray}
where $\V{g}$ is some embedding function from $\mathcal{Y}$ and $D_{\V{g}} (\V{y}_i, \V{y}_j)= \| \V{g}(\V{y}_i)- \V{g}(\V{y}_j)\|_2 $. Specifically the stress values for LOCA, and the scaled versions of DM and A-DM are $1.5 \cdot 10^{-5}, 0.03$ and $0.002$, respectively. As evident from the stress values and from Fig. \ref{fig:mushroom_xy}, LOCA provides an embedding that is isometric to $\mathcal{X}$ (up to an orthogonal transformation and shift). 

\subsection{The encoder is observed to extend reliably to unseen samples}
\label{sec:out_of_sample}

In the next experiment, we evaluate the out-of-sample extension capabilities of LOCA. The experiment is based on the same nonlinear transformation described in Section \ref{sec:experimets_isometry}. We sample $N=2,000$ points from a partial region of the latent representation $\mathcal{X}$, specifically described by $[0,1]^2 \backslash [0.1,0.9]^2$. In Fig. \ref{fig:mushroom_oos}, we present the framed sampling regions along with the corresponding observed framed regions in $\mathcal{Y}$ (see black and green frames \ref{fig:oos_xFrame} and \ref{fig:oos_yFrame}). To generate the ``bursts'' we follow the setting presented in Section \ref{sec:experimets_isometry} and refer to them as our training set. The test set is defined by an additional $2\cdot 10^4$ samples generated as in Section \ref{sec:experimets_isometry} from $[-.025,1.025]^2$ in $\mathcal{X}$ pushed forward by $\V{f}_1$.

In Fig. \ref{fig:oos_Frame_loca}, we quantify the interpolation and extrapolation capabilities of LOCA by presenting the extracted embedding along with the corrected frame. To further evaluate the quality of this embedding we compare the pairwise distances in \ref{fig:oos_Frame_loca_dists} (as described in Section \ref{sec:experimets_isometry}). This comparison (presented in Fig. \ref{fig:oos_Frame_loca_dists}) supports the benefits of using LOCA for extending the representation to unseen samples. To stress it even more, the actual stress values of LOCA in the interpolation region, on the frame and in the extrapolation region are all approximately $10^{-4}$.

\begin{figure*}[!htb]
 \centering
\subfigure[ ]{\label{fig:oos_latent_inter_x} \includegraphics[width=0.23\textwidth] 
{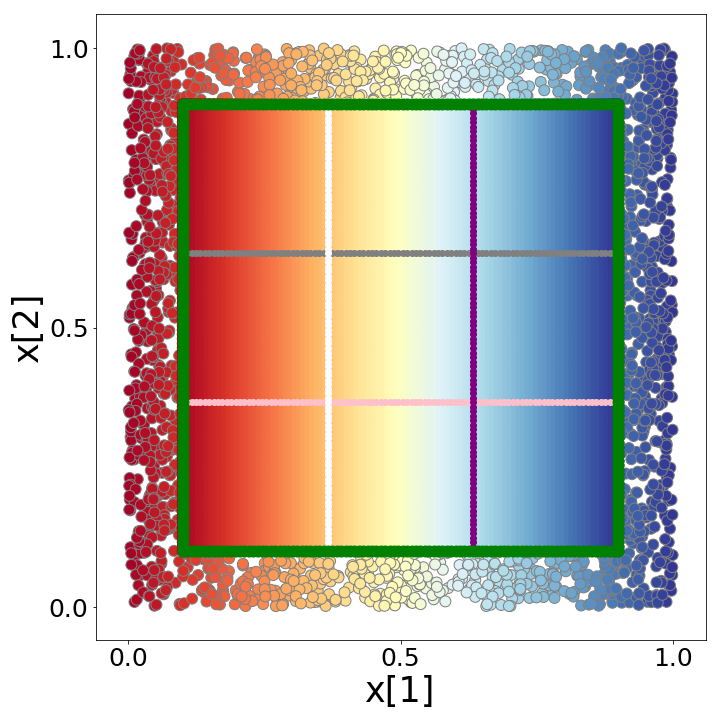}}
\subfigure[ ]{\label{fig:oos_latent_inter_y} \includegraphics[width=0.23\textwidth]   
{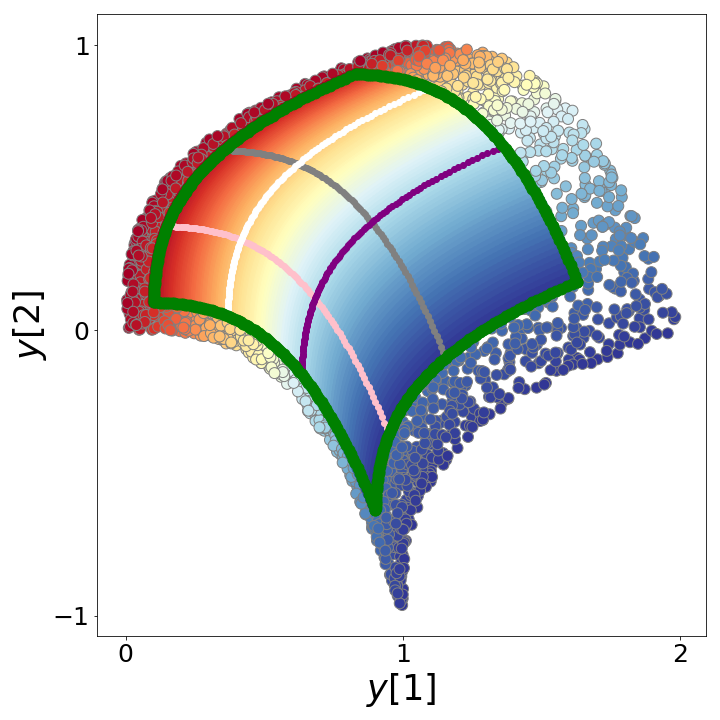}}
\subfigure[ ]{\label{fig:oos_latent_inter_loca_code} \includegraphics[width=0.23\textwidth]   
{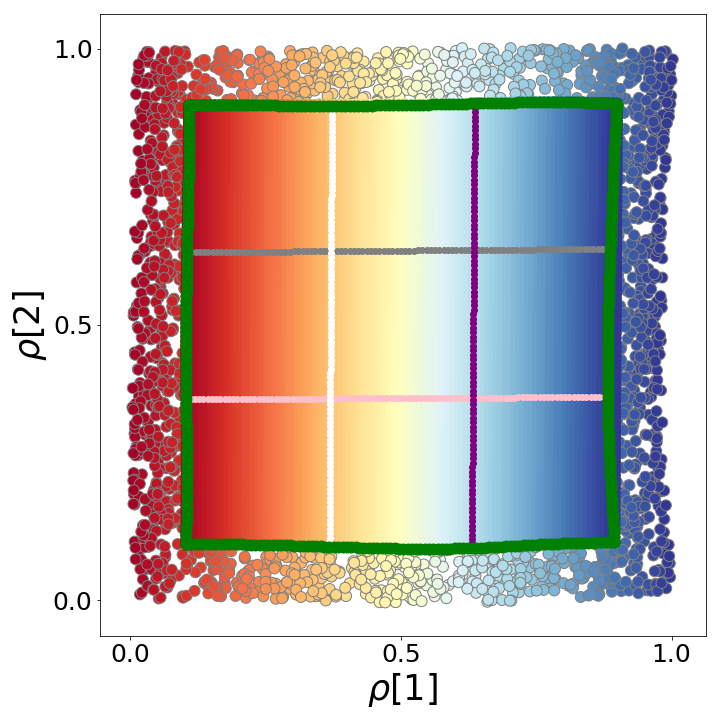}}
\subfigure[ ]{\label{fig:oos_latent_inter_loca_recon} \includegraphics[width=0.23\textwidth]   
{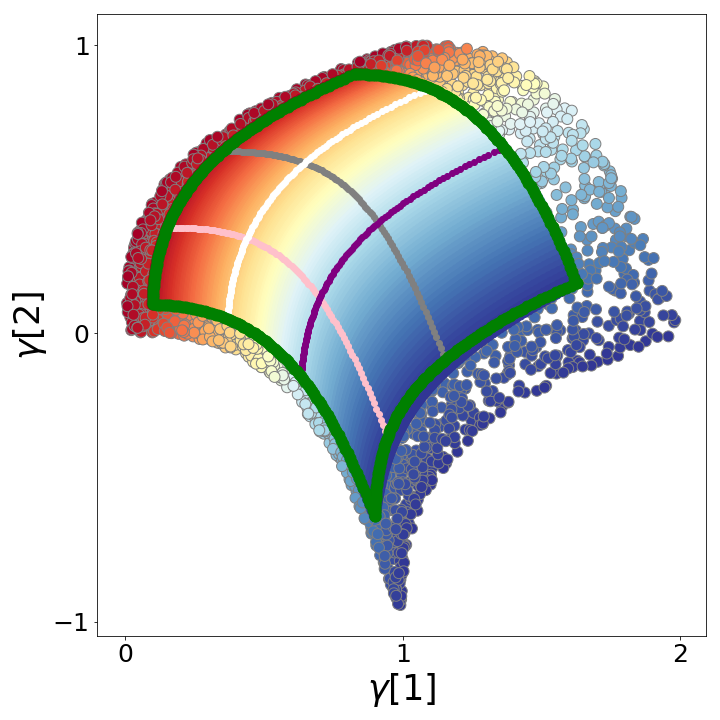}}
\caption{Evaluating the out-of-sample reconstruction capabilities of LOCA. Here, we attempt to generate new point in the ambient space by performing linear interpolation in the embedding space. A description of the linear interpolation appears in \ref{sec:interpolation}. \subref{fig:oos_latent_inter_x} The inaccessible latent space, the points surrounding the green frame are the training samples. Interpolation is performed horizontally and vertically between points on the green frame, see for example the $4$ colored lines. \subref{fig:oos_latent_inter_y} The pushed forward data from \subref{fig:oos_latent_inter_x} based on the non linear function $\V{f}_1$ described in \eqref{eq:experiment_tranform_mush}.
\subref{fig:oos_latent_inter_loca_code} Recovered calibrated embedding $\encoder$. The training samples in the frame and green boarder are embedded using LOCA. Within the embedded green boarder we perform an additional linear interpolation, using the same corresponding pairs as were used in the latent space. For example, see the $4$ horizontal and vertical colored lines. \subref{fig:oos_latent_inter_loca_recon} The pushed forward data from \subref{fig:oos_latent_inter_loca_code} by the decoder ($\V{\gamma}$) learned by LOCA. This experiment demonstrates that LOCA learns a decoding function that is consistent with the unknown transformation, even in a regime that is not covered by training samples.
}
    \label{fig:mushroom_interpolation}
\end{figure*}

\subsection{The decoder is observed to extend reliably to unseen samples}
\label{sec:interpolation}
In this experiment we evaluate the out-of-sample capabilities of LOCA's decoder. While in \ref{sec:out_of_sample} we trained LOCA and evaluated the quality of the encoder on unseen data, here we focus on the performance of the decoder. Specifically, we apply the decoder to unseen samples from the embedding space. Each unseen sample in the embedding space is created using linear interpolation. We now provide the exact details of this evaluation.

For this interpolation experiment, we use the same LOCA model trained in \ref{sec:out_of_sample} on the framed data. We further generate $N=400$ points in the interior boundary of the frame, represented by the green dots in \figref{fig:mushroom_interpolation}. Next, we perform linear interpolation between horizontal and vertical pairs, see for example the colored lines in  \figref{fig:oos_latent_inter_x}. 
The data is then pushed forward using the nonlinear transformation described in \eqref{eq:experiment_tranform_mush}, as shown in  \figref{fig:oos_latent_inter_y}. We embed the training samples along with the green frame using LOCA; a calibrated version of the embedding space is shown in \figref{fig:oos_latent_inter_loca_code} (up to a shift and an orthogonal transformation). Then, we perform an additional interpolation in the embedding space using the same corresponding pairs as were used in the latent space (see \figref{fig:oos_latent_inter_x}). Finally, we apply the decoder to the embedding of the training samples and to the newly interpolated samples, these are presented in \figref{fig:oos_latent_inter_loca_recon}.
As evident in this figure, the reconstructed points faithfully capture the mushroom shaped manifold. The mean squared error between the push-forward interpolated points and the decoded interpolated points is $2.3\cdot 10^{-4}$, with a standard deviation of $2.4\cdot 10^{-4}$. This experiment demonstrates that LOCA may be also used as a generative model, by reconstructing new points generated using interpolation in the embedding space.

\begin{figure}[!htb]
\centering
    \subfigure[ ]{\label{fig:stereographical_illustration} \includegraphics[width=0.32\textwidth, trim={0 7cm 0 7cm},clip] {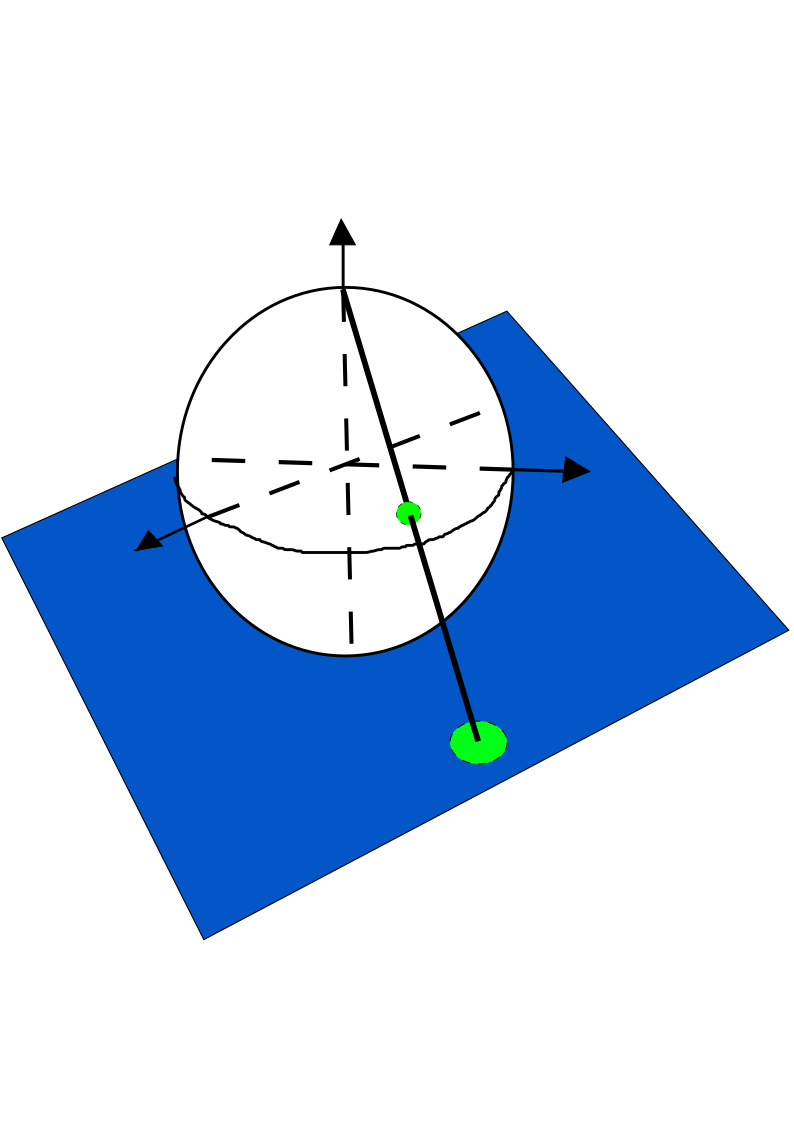}}
    \subfigure[ ]{\label{fig:stereographical_origData} \includegraphics[width=0.32\textwidth] {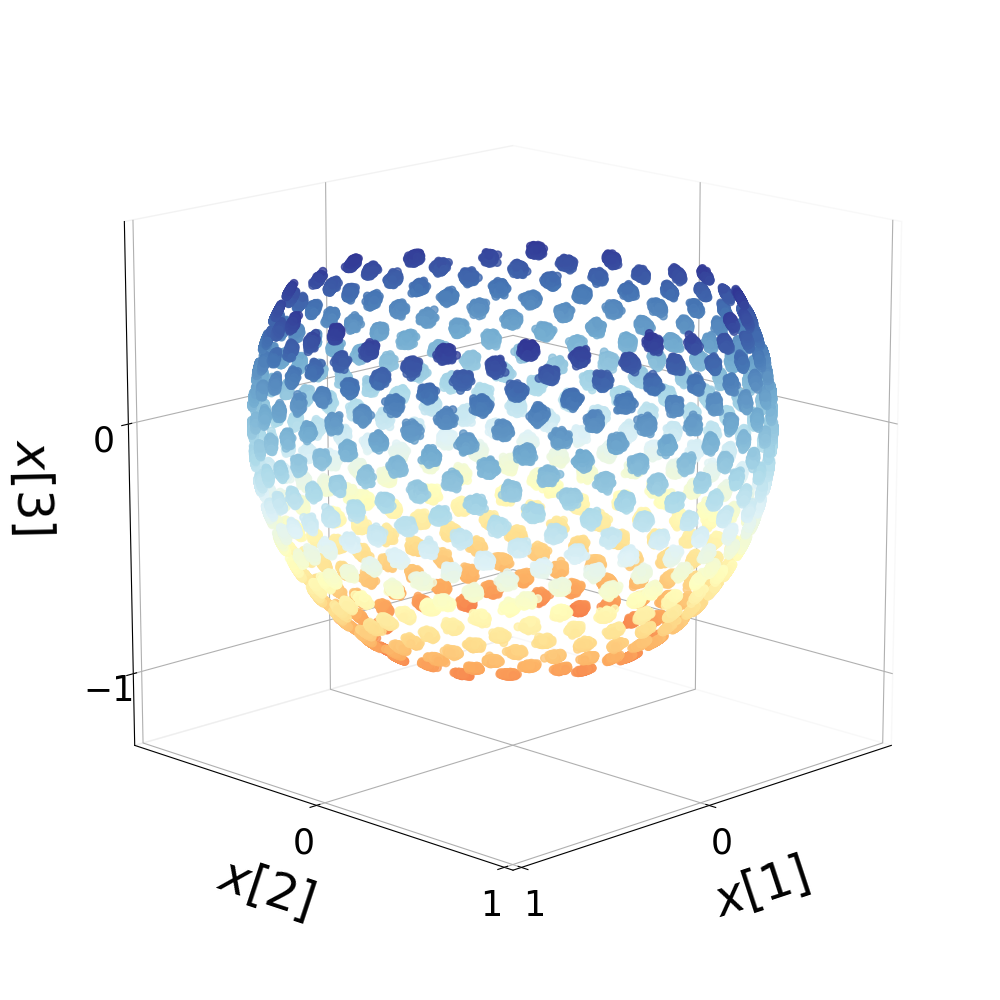}}
    \subfigure[ ]{\label{fig:stereographical_observedData} \includegraphics[width=0.32\textwidth] {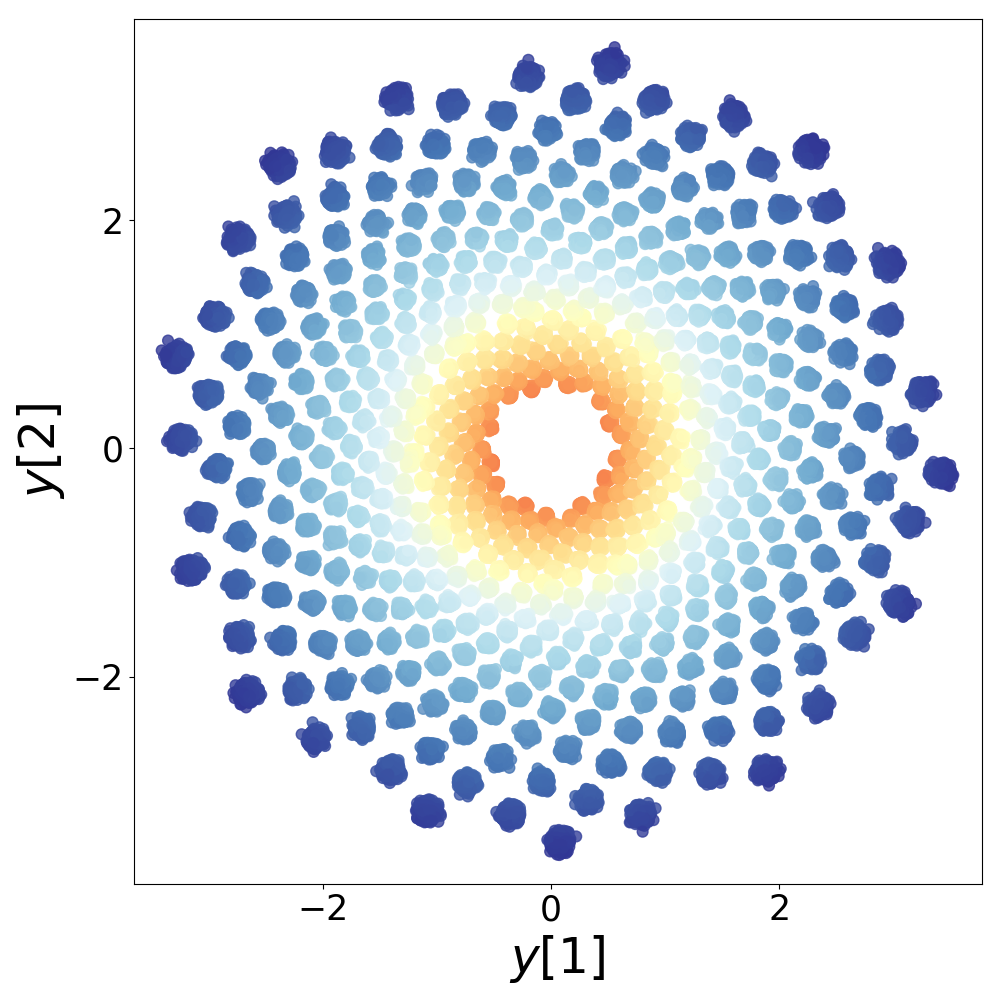}}
    \subfigure[]{\label{fig:stereographical_embeddingLoca}
    \includegraphics[width=0.32\textwidth]
    {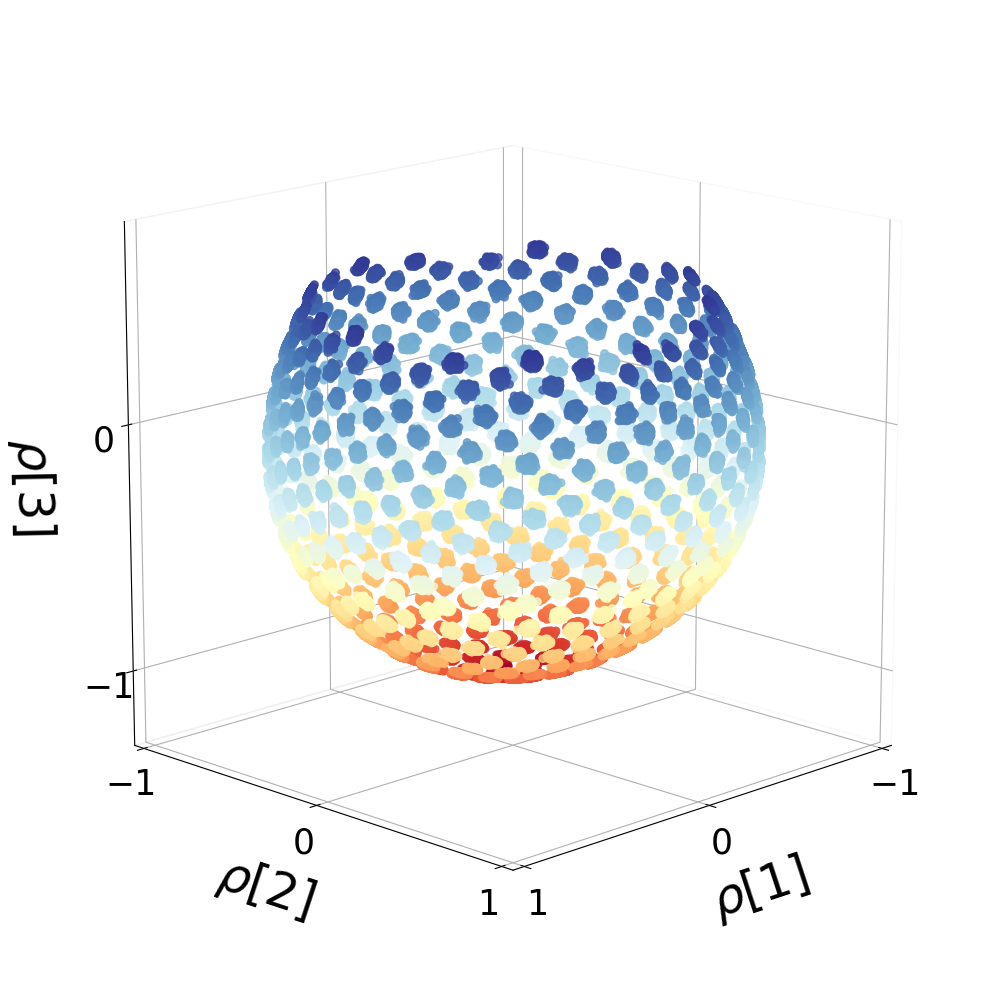}}
    \subfigure[ ]{\label{fig:stereographical_results} 
    \includegraphics[width=0.32\textwidth] {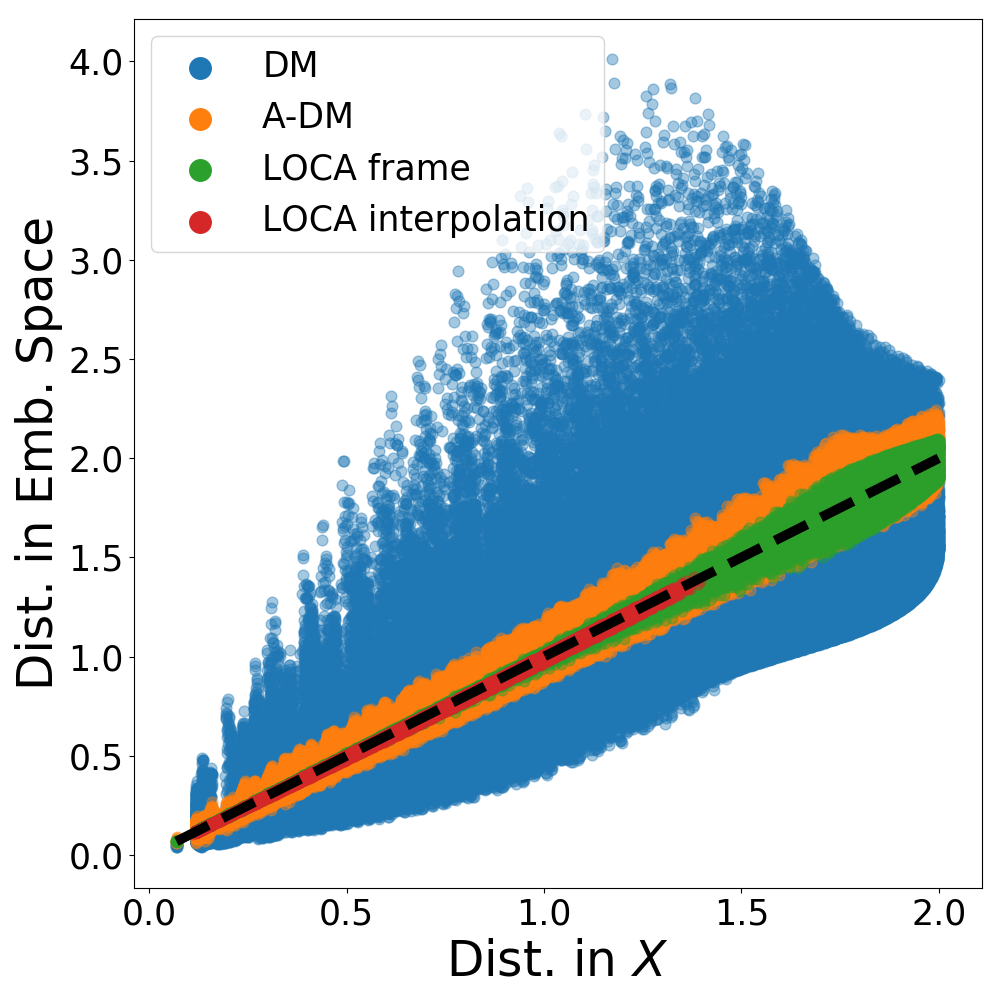}}
    \caption{The stereographic projection experiment (see description in Section \ref{sec:Loca_curved_manifold}). \subref{fig:stereographical_illustration}: A schematic illustration of the stereographic projection generating the data. \subref{fig:stereographical_origData}: The original latent representation of the ``bursts'' employed. \subref{fig:stereographical_observedData}: The 2-dimensional observations of the ``bursts" created using the stereographic projection. The plot contains only the training data, meaning points that satisfy $\alpha\in (\pi/3,5\pi/6)$, leaving a ``hole" at the south pole.  \subref{fig:stereographical_embeddingLoca}: The 3-dimensional embedding of these training data, with the missing lower cap ($\alpha\in (\pi/3,\pi]$).
    The color represents the value of $\alpha$ of each point as defined in \eqref{eq:stereographical_x}. The colors used in \subref{fig:stereographical_origData}- \subref{fig:stereographical_embeddingLoca} correspond to the spherical angle $\alpha$ defined in \eqref{eq:stereographical_x}. \subref{fig:stereographical_results} The Euclidean distances between pairs of points in the original, 3-dimensional latent space versus the corresponding Euclidean distance in the embedding space. Here we compare distances based on the training region (frame) as well as the unseen test region where $\alpha\in (\pi/6,\pi]$ (interpolation).}
    \label{fig:stereographical_data}

\end{figure}
\subsection{LOCA on a curved manifold}
\label{sec:Loca_curved_manifold}
Here, we examine a more challenging configuration, generalizing our original Euclidean problem setting. The latent space is now taken to be a $k-$dimensional manifold that resides in $\R^d$, where $d>k$ and $d$ is the minimal dimension required to embed the manifold in a Euclidean space isometrically. Interestingly, we consider an observation process such that the observation dimension, $D$, is smaller than $d$. To clarify, this means that the measurement process can involve projections to a lower dimension.

We consider a manifold that covers three quarters of a $2$-dimensional unit sphere in $\R^3$, where the training points admit the following form
\begin{eqnarray}
\label{eq:stereographical_x}
\V{x}= \left(\begin{array}{c} \sin(\alpha)\cos(\beta) \\ \sin(\alpha) \sin(\beta) \\ \cos(\alpha) \end{array}  \right) \qquad \beta\in [0,2\pi), \alpha \in [\pi/3, \pi].
\end{eqnarray}

The manifold is embedded in $\R^3$ but has an \textit{intrinsic dimension} of $2$. This requires us to revisit our definition of ``bursts'', discussed in \eqref{eq:ambient_clouds}. Specifically, we assume that the bursts are confined to the manifold. 
Here, we approximate this constraint in the form of random variables $\V{Z}_i$ obtained using a local isotropic Gaussian with a two dimensional covariance $\sigma^2 \V{I}_2$, defined on the tangent plane to the point.  

We consider $N=491$ states of the system $\V{x}_i,i=1,..,N,$ which are generated on a uniform grid using the ``Fibonacci Sphere'' sampling scheme \cite{swinbank2006fibonacci} for points with $\alpha\in [\pi/3,5\pi/6]$. We define each ``burst'' $\V{X}_i$ using $M=400$ points sampled from our two-dimensional isotropic Gaussian defined by the tangent plane around $\V{x}_i$ with $\sigma =0.01$. Now, in order to create the observed samples $\V{y}$ we apply the stereographic projection to $\V{x}$ by projecting each point from $\mathcal{X}$ onto a two dimensional space defined by:
\begin{eqnarray}
\V{y} = \left( \begin{array}{c} \frac{x[1]}{1-x[3]} \\ \frac{x[2]}{1-x[3]} \end{array}\right).
\end{eqnarray}
The transformation can be thought of as a projection onto the plane $\mathbb{R}^2\times \{1\}$; an illustration of the stereographic projection appears in Fig. \ref{fig:stereographical_illustration}. The training ``bursts'' in the latent space and the observed space appear in Figs. \ref{fig:stereographical_origData} and \ref{fig:stereographical_observedData}, respectively.

We apply DM, Anisotropic DM, and LOCA to embed the data in a $3$-dimensional space. The difference between the pairwise Euclidean distances (see description in Section \ref{sec:experimets_isometry}) in each embedding space and the original Euclidean distances along with the extracted embeddings are described in the Supplementary Information. The stress values for LOCA, and the scaled DM and A-DM on the training data are $10^{-3},0.18$ and $6\cdot 10^{-3}$, respectively. In order to examine the interpolation capabilities of LOCA, we generate $55$ points using the "Fibonacci Sphere" that satisfies $\alpha\in (5\pi/6,\pi]$. Using the trained model of LOCA we embed these data and get that the stress value is $10^{-4}$ . Fig. \ref{fig:stereographical_data} demonstrates that LOCA can well approximate an isometry even if the dimension of the observations is lower than the minimal embedding dimension needed for the isometry, i.e. $k>D$.

\section{Applications}
\subsection{Flattening a curved surface}
\label{sec:flatten_surface}
Our first application is motivated by \cite{han2018robust}, in which the authors propose a method for estimating the $3$-dimensional deformation of a $2$-dimensional object. They focus on the task of autonomous robotic manipulation of deformable objects. Their method uses a stream of images from a RGB-D camera and aligns them to a reference shape to estimate the deformation.

We explore the applicability of LOCA for the task of estimating a deformation based on a $2$-dimensional projection of an object, without using any depth feature. We print a black square-shaped grid of $N=2500$ points of interest; at each location, we generate a ``burst'' with $M=50$ samples drawn from a Gaussian with $\sigma=0.01$. We manually squeeze the printed square and photograph the deformed object from above. The image of the original squared object along with a $2$-dimensional snapshot of the deformed object appears in Fig. \ref{fig:paper_train_image}. 
This experiment complements our motivating example presented in Fig. \ref{fig:motivation}.

To define the anchor points $\V{y}_i$ along with corresponding ``bursts'', we first identify the locations of all points by applying a simple threshold filter to the image. Then, we identify the ``bursts'' by applying the Density-Based Spatial Clustering of Applications with Noise (DBSCAN) \cite{dbscan}. In Fig. \ref{fig:paper_train_image} we present the identified groups of points (black). Note that some ``bursts'' are lost in this process, as there is nearly no gap between them in the deformed shape. Here, the parameter $\sigma^2$ for the whitening loss  \eqref{eq:white_loss} is estimated using the median of the first eigenvalue of the ``bursts'' covariances. We apply LOCA and extract the embedding $\encoder$. In Fig. \ref{fig:paper_embed_image}, we present a calibrated version (scaled rigid transformation) of the embedding, $\tilde{\V{\rho}}$, overlaid on the latent representation. The transformation is found by minimizing the mean squared error between the underlying representation and the extracted embedding of the $4$ corners of the square. This experiment demonstrates that LOCA corrects the unknown deformation based on the {\em estimated} ``bursts''. 

\begin{figure}[!htb]
    \centering
    \subfigure[]{\label{fig:paper_train_image}
        \includegraphics[width=0.45\textwidth]{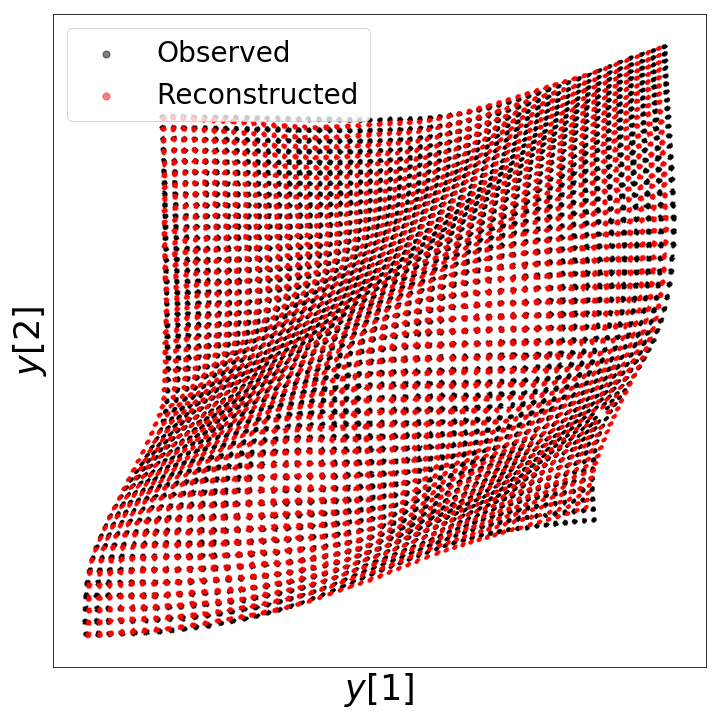}}
    \subfigure[] { \label{fig:paper_embed_image}
        \includegraphics[width=0.45\textwidth]{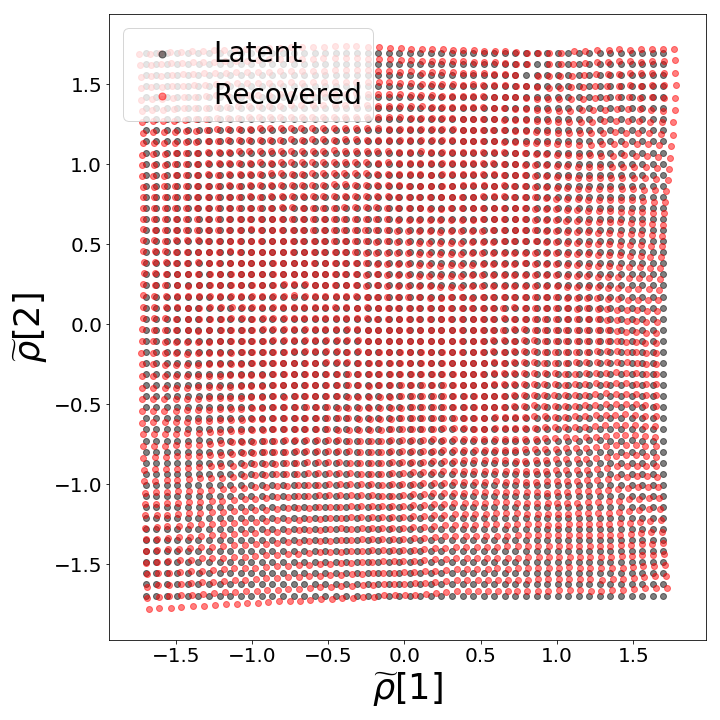}}
    \caption{A LOCA embedding can flatten a deformed object using estimated ``bursts''. \subref{fig:paper_train_image} The input training samples used by LOCA (black) and the reconstructed points (red). \subref{fig:paper_embed_image} The calibrated embedding (using an orthogonal transformation and a shift) of the deformed object using LOCA (red) and the underlying representation of points (black). As in the synthetic examples, we use calibration only for visualization purposes. Here, LOCA manages to correct the deformation of the local ``bursts'', and thus learns a function that approximately uncovers the latent structure of the object. }
    \label{fig:flatten_image_res}
\end{figure}

\subsection{Application to Wi-Fi localization}
\label{sec:wifi}
Here, we evaluate LOCA for the task of geographical localization based on mobile devices. The localization problem involves estimating the geographic location of a receiver device based on signals sent from multiple Wi-Fi transmitters.
This problem has been addressed by modeling the strength of the signal in time and space, or based on fingerprints of the signal learned in a supervised setting \cite{jaffe2014single,alikhani2017fast}. We address the problem in an unsupervised fashion by applying the proposed LOCA algorithm without employing any physical model.
 
The experiment is performed by simulating the signal strength of $L=\wifiAccessPointCountNumeric$ Wi-Fi transmitters at multiple locations across a model of a room, where each transmitter uses a separate channel. The room is modeled based on a simplified floor plan of the fourth floor of MIT's Ray and Maria Stata Center. We refer to the two-dimensional representation of the room as $\mathcal{X}\subset \mathbb{R}^2$; a schematic of the floor plan with $600\times 1000$ pixels appears in Fig. \ref{fig:wifi-full} (black line). The $L=\wifiAccessPointCountNumeric$ Wi-Fi transmitters are randomly located across the floor plan; we denote each of these locations by $\V{t}_{\ell}\in \mathbb{R}^2$, for any $\ell\in \{1,...,\wifiAccessPointCountNumeric\}$. Next, we sample $\V{x}_i,i=1,..., N$, using $N=\numSensorReadings$ anchor points distributed uniformly over $\mathcal{X}$ and define the amplitude of each measured Wi-Fi signal using a radial basis function (RBF). The RBF decay is monotonic in the distance between the transmitter and the measurement location, so that the amplitude at point $x_i$ of the signal of transmitter $\ell$ is $y_{i,\ell}=\exp(-\|\V{x}_i-\V{t}_{\ell}\|_2^2 / \epsilon ^ 2)$, where $\epsilon =\wifiTransmissionDistance$ pixels. Here the ``bursts'' will be defined by a circle of $M=6$ receivers equally spaced at a radius of $r={\sensorArrayRadiusPixels}$ pixels around each anchor point $\V{x}_i$: these 6 receivers model a circular sensor array as the measurement device.
 
 Next, we apply LOCA and embed the observed vectors of multi-channel amplitudes into a $2$-dimensional space. To demonstrate the performance of LOCA  we calibrate the LOCA embedding to the ground truth floor plan using a shift and scaled orthogonal transformation, as done in \ref{sec:flatten_surface} but using all the training data.  In Fig. \ref{fig:wifi-full} we present the scaled, calibrated two-dimensional embedding $\tilde{\encoder}$ with the locations of the transmitters and anchor points.

\begin{figure}[!h]
    \centering
    \includegraphics[width=.9\textwidth]{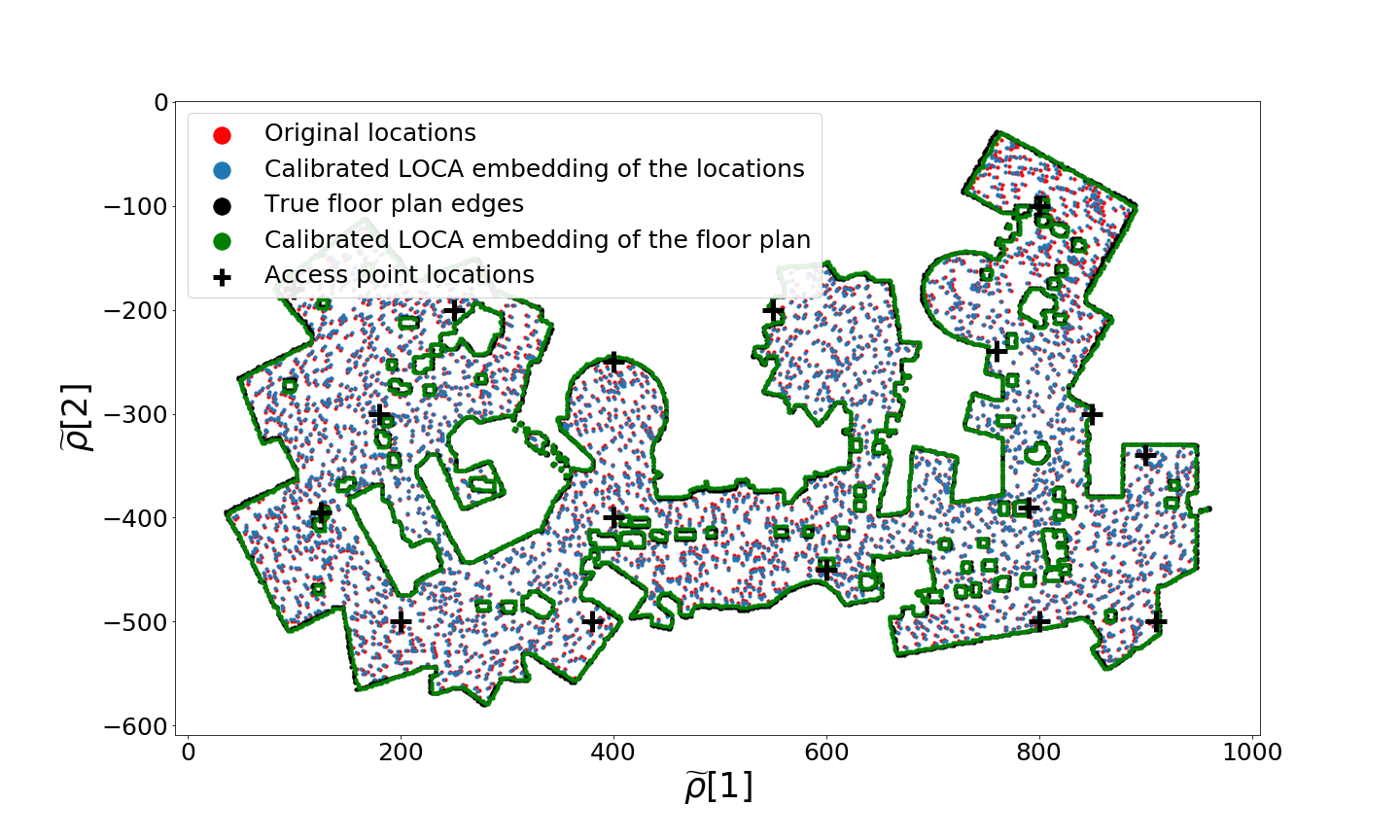}
    \caption{Application of LOCA to Wi-Fi localization. We use a floor plan model based on the fourth floor of MIT's Ray and Maria Stata Center. The edges of the ground truth model appear in black. We simulate $L=17$ Wi-Fi access points (transmitters), which are presented as black crosses. We use $N=4000$ locations depicted as red dots with corresponding $M=6$ ``burst samples'' around them (modeling a circular antenna array). To demonstrate that LOCA's embedding is coherent with the latent representation, we calibrate the embedding to the true floor plan, see blue dots and green line. 
}
    \label{fig:wifi-full}
\end{figure}

\section{Discussion}

We propose a method that extracts canonical data coordinates from scientific measurements. These approximate an embedding that is isometric to the latent manifold. We are assuming a specific, broadly applicable stochastic sampling strategy, and our proposed method corrects for unknown measurement device deformations. 
Our method constructs a representation that "whitens" (namely, changes to multivariate z-scores) groups of local neighborhoods, which we call ``bursts''. We impose additional constraints to patch together the locally whitened neighborhoods, ensuring a smooth global structure. Finally, the method is implemented using a neural network architecture, namely an encoder-decoder pair, which we name LOcal Conformal Autoencoder (LOCA).

 The method can be summarized as follows. 
 (i) We collect distorted neighborhoods of a fixed size of data samples.(ii) We embed/encode the data in the lowest dimensional Euclidean space so that these neighborhoods are standardized or z-scored.(iii) The data is decoded from embedding space to original measurements, enabling interpolation and extrapolation. (iv) LOCA is invariant to the measurement modality (approximately to second order, and modulo a rigid transformation). (v) Under scaling consistency for samples drawn from a Riemannian manifold, the encoder can approximate an isometric embedding of the manifold. 
 
 From an implementation perspective, our method is simpler than existing manifold learning methods, which typically require an eigen-decomposition of an $N$-by-$N$ matrix ($N$ being the number of samples). Indeed, existing implementations of deep neural networks enable a single developer to produce fast, reliable, GPU-based implementations of LOCA.
 
We provided solid empirical evidence that, if the deformation is invertible, then LOCA extracts an embedding that is isometric to the latent variables. Moreover, LOCA exhibits intriguing, indeed promising interpolation and extrapolation capabilities. To motivate the benefits of using LOCA, we used two potential applications. First, we apply a $3$-dimensional deformation to a printed object and demonstrate that LOCA manages to invert the deformation without using any assumptions on the object's structure. Finally, using Wi-Fi generated signals from multiple locations, we show that the LOCA embedding can be quantititatively correlated with the true locations of the received signals.

Our method relies on the ``bursts'' measurement model. As shown in Lemma \ref{lem:cov_jacobian}, the covariances of the bursts can be used to estimate the Jacobian of the unknown measurement function. Alternatively, we can replace this estimation with any other type of measurement strategy informative enough to estimate the local Jacobian of the measurement function.

\section*{Acknowledgements}

This work was partially supported by the DARPA PAI program (Agreement No. HR00111890032, Dr. T. Senator). This material
is based upon work supported by, or in part by, the U. S. Army Research Laboratory and the U. S. Army Research Office under contract/grant number W911NF1710306. E.P. has been partially supported by the Blavatnik Interdisciplinary Research Center (ICRC), the Federmann Research Center (Hebrew University) and Israeli Science Foundation research grant no. 1523/16.

\bibliography{main}
\bibliographystyle{ieeetr}

\begin{appendices}

\section{Proof of Lemma}
\label{sec:proof of lemma}
\label{sec:Supplementary Material}
\begin{proof}[Proof of lemma \ref{lem:cov_jacobian}]
For a sufficiently small $\sigma$ and any $\V{x}$ such that $\|\V{x}-\V{x}_i\|_2^2= O(\sigma^2)$ we can express $\V{g(x)}$ by 
$$ \V{g}(\V{x})=\V{g}(\V{x}_i)+\V{J_{g}}(\V{x}_i)(\V{x}-\V{x}_i)+O(\|\V{x}-\V{x}_i\|_2^2);$$ 
here we use the smoothness of $\V{g}$. 
Hence, we can express the covariance of the random variable $\V{Y}_i=\V{g}(\V{X}_i)$ by

\begin{eqnarray*}
 \V{C}(\V{Y}_i)&=& \mathbb{E} [\left(\V{Y}_i-\mathbb{E}[\V{Y}_i]\right)\left(\V{Y}_i-\mathbb{E}[\V{Y}_i]\right)^T ]\\
 &=& \mathbb{E} [\left(\V{Y}_i-\V{g}(\V{x}_i) +O(\sigma^2) \right)\cdot \\
 & &\left(\V{Y}_i-\V{g}(\V{x}_i) +O(\sigma^2)\right)^T ]\\
 &=&
 \V{J}_g (\V{x}_i) \mathbb{E} [ (\V{x} -\V{x}_i)(\V{x} -\V{x}_i)^T] \V{J}^T_g (\V{x}_i) + O(\sigma^4) \\
 &=& \sigma^2 \V{J}_{g} (\V{x}_i) \V{J}^T_{g} (\V{x}_i) + O(\sigma^4),
\end{eqnarray*}
where 
\begin{eqnarray*}
\mathbb{E}[\V{Y}_i]&=& \mathbb{E}[\V{g}(\V{X}_i)] \\
&=& \mathbb{E}[\V{g}(\V{x}_i) +\V{J_{g}} (\V{x}_i) (\V{X}_i-\V{x}_i) +(\|\V{X}_i-\V{x}_i\|_2^2)]\\
&=&\V{g}(\V{x}_i) +O(\sigma^2).
\end{eqnarray*}
\end{proof}
\section{Methods Description}
\label{sec:dm_adm_description}

Here we provide a description and implementation details for Diffusion maps (DM) \cite{Dmaps} and Anisotropic Diffusion maps (A-DM) \cite{nonlinearICA}.
\subsection{Diffusion maps}
\label{sec:dm_method}
Diffusion maps (DM) \cite{Dmaps} is a kernel based method for non linear dimensionality reduction. The method relies on a stochastic matrix built using a kernel $\V{K}: \mathcal{M}_Y \times \mathcal{M}_Y \rightarrow \mathbb{R}$. The stochastic matrix can be viewed as a fictitious random walk on the graph of the data. The reduced representation is obtained via an eigendecomposition of the stochastic matrix. The DM construction is summarized in the following steps:
\begin{enumerate}

\item Define a kernel function  \begin{math} {{\cal{K}} : \V{Y}\times{\V{Y}}\longrightarrow{\mathbb{R}}  }
\end{math}, such that $\V{K} \in {\mathbb{R}^{N \times N}}$ with elements $K_{i,j}={\cal{K}}(\V{y}_i,\V{y}_j)$, where $\V{K}$ is symmetric, positive semi-definite and non-negative.
Here, we focus on the common Radial Basis kernel defined for any $\V{y}_i$ and $\V{y}_j$ as
	\begin{equation}\label{eq:GKernel}
	{\mathcal{K}}(\V{y}_i,\V{y}_j)\defeq K_{i,j}=\exp\left( {-\frac{||\V{y}_i-\V{y}_j||^2}{2
			\epsilon}  }\right),i,j\in\{1\ldots N\},
	\end{equation}
	where $\epsilon$ is a kernel bandwidth (see more details below).

\item Row normalize $\V{K}$ 
    \begin{equation}
    \label{eq:DefP}
    \V{P}\defeq\V{D}^{-1}\V{K}\in\mathbb{R}^{N\times N},
    \end{equation}
  where the diagonal matrix $\V{D}\in\mathbb{R}^{N\times N}$ is defined as $D_{i,i}=\sum_jK_{i,j}$. $\V{P}$ can be interpreted as the matrix of transition probabilities of a Markov chain on $\V{Y}$, such that $\left[(\V{P})^t\right]_{i,j}\defeq p_t(\V{y}_i,\V{y}_j)$ (where $t$ is an integer power) describes the implied probability of transition from point $\V{y}_i$ to point $\V{y}_j$ in $t$ steps.

\item{Define the embedding for the dataset $\V{Y}$ by
	\begin{equation}\label{eq:Psi}{ \V{\Psi}^{t}: {(\V{y}_i)}:   \V{x}_i
		\longmapsto \begin{bmatrix} { \lambda_1^{t}\psi_1(i)} , {
			\lambda_2^{t}\psi_2(i)} , { \lambda_3^{t}\psi_3(i)} , {.} {.} {.}
		,
		
		{\lambda_{d}^{t}\psi_{d}(i)}\\
		
		\end{bmatrix}^T \in{\mathbb{R}^{d}} },
	\end{equation}

}
\end{enumerate}
where $\lambda_i$ and $\V{\psi}_i$ are the $i$-th eigenvalue and right eigenvector of the matrix $\V{P}$.

It is important to properly tune the kernel scale/bandwidth $\epsilon$, which determines the scale of connectivity of the kernel $\V{K}$. 
Several studies have suggested methods for optimizing $\epsilon$ in DM (e.g. \cite{Amit2,Keller,zelnik,epsilon}). 
Here, we use the max-min approach initially suggested in \cite{Keller} where the scale is set to
\begin{equation} \label{eq:MaxMin}
\epsilon_{\text{MaxMin}}= \underset{j}{\max} [ \underset{i,i\neq j}{\min} (||\V{y}_i-\V{y}_j||^2)],i,j=1,...N.
\end{equation} The max-min aims for a scale that ensures that all points are connected to at least one other point.
\subsection{Anisotropic Diffusion maps}
The Anisotropic Diffusion maps (A-DM) proposed in \cite{nonlinearICA} effectively replace the Euclidean distance in \eqref{eq:GKernel} by a (joint) local Mahalanobis distance. This local Mahalanobis distance between observed anchor points $\V{y}_i$ and $\V{y}_j$ is computed using the observed ``short bursts'' by
\begin{equation}
    \| \V{y}_i-\V{y}_j\|^2_M \defeq \frac{1}{2}(\V{y}_i-\V{y}_j)^T[\V{C}_i^{\dagger}+\V{C}_j^{\dagger}](\V{y}_i-\V{y}_j),
\end{equation}
where $\V{C}_i^{\dagger}$ is the generalized inverse of the sample covariance of the $i$-th observation burst.
This joint local Mahalanobis distance is used to compute the Anisotropic diffusion kernel 
\begin{equation}
    	{\widetilde{\mathcal{K}}}(\V{y}_i,\V{y}_j)\defeq \widetilde{K}_{i,j}=\exp\left( {-\frac{||\V{y}_i-\V{y}_j||^2_M}{2
			\epsilon}  }\right),i,j\in\{1\ldots N\};
			\end{equation}
the scaling parameter is again optimized using \eqref{eq:MaxMin} but now based on the Mahalanobis metric.
Next, we follow steps (2) and (3) in DM (see Section \ref{sec:dm_method}) and compute the A-DM embedding $\V{\phi}$ using the right eigenvectors of the normalized kernel. 

\section{Numerical Experiment Setting.}
\label{sec:details}
Here we describe the implementation details of each numerical experiment undertaken. The architectures we used are based on fully connected layers, with an activation function after each layer, except for the last two layers.
\begin{itemize}
\item First experiment details  (\ref{fig:mushroom_xy})
    \begin{itemize}
        \item Neural net architecture- \\
        Encoder- Neurons in each layer:  [50,50,2,2],  Activation function: tanh.\\
        Decoder- Neurons in each layer: [50,50,2,2], Activation function: tanh.
        \item Data- amount clouds: $N=2000$, cloud size: $M=200$, noise standard noise: $\sigma=1e-2$.
        \item Neural net training- batch size: 200 clouds, amount training clouds: 1800, amount validation clouds: 200.
    \end{itemize}
\item Second and third experiment details (\ref{fig:mushroom_oos},\ref{fig:mushroom_interpolation})
 \begin{itemize}
       \item Neural net architecture- \\
        Encoder- Neurons in each layer:  [50,50,2,2],  Activation function: tanh.\\
        Decoder- Neurons in each layer: [50,50,2,2], Activation function: tanh.
        \item Data- amount clouds: $N=2000$, cloud size: $M=200$, noise standard noise: $\sigma=1e-2$.
        \item Neural net training- batch size: 200 clouds, amount training clouds: 1800, amount validation clouds: 200.
\end{itemize}
\item Forth experiment details (\ref{fig:stereographical_data})
 \begin{itemize}
        \item Neural net architecture- \\
        Encoder- Neurons in each layer:  [100,100,3,3], Activation function: tanh.\\
        Decoder- Neurons in each layer:  [100,100,2,2], Activation function: leaky relu.
        \item Data- amount clouds: $N=546$, cloud size: $M=400$, noise standard noise: $\sigma=1e-2$. 
        
        The points on the sphere were generated using the "Fibonacci Sphere" mechanism with $800$ points. The points that did not satisfy constraints were left out.
        \item Neural net training- batch size: $50$ clouds, amount training clouds: $491$, amount validation clouds: $55$.
        \item Generation of a cloud around a given point on the unit sphere- Let $(\alpha,\beta)$ be its polar representation. We sample a cloud of points in the polar space using $\mathcal N_2((\pi/2,0), \sigma^2 \V{I}_2)$. Next, we find some orthogonal transformation that maps $(\pi/2,0)$ to $(\alpha,\beta)$ in the $3$-dimensional Cartesian space, and apply it to each sampled point.
\end{itemize}
\item Flattening a curved surface experiment details(\ref{sec:flatten_surface})
\begin{itemize}
        \item Neural net architecture- \\
        Encoder- Neurons in each layer:  [200,200,2,2], Activation function: tanh.\\
        Decoder- Neurons in each layer:  [200,200,2,2], Activation function: leaky relu.
        \item Data- amount clouds: $N=2500$, cloud size: $M=60$.
        
        \item Neural net training- batch size: $250$ clouds, amount training clouds: $2250$, amount validation clouds: $250$
        
\end{itemize}

\item Wi-Fi localization experiment details (\ref{sec:wifi})
\begin{itemize}
        \item Neural net architecture- \\
        Encoder- Neurons in each layer:  [200,200,3,3], Activation function: tanh.\\
        Decoder- Neurons in each layer:  [200,200,2,2], Activation function: leaky relu.
        \item Data- amount clouds: $N=4000$, cloud size: $M=6$.
        
        \item Neural net training- batch size: $200$ clouds, amount training clouds: $3600$, amount validation clouds: $400$.
        \item Generation of a cloud around a given point - sample the circle every $\pi/3$ starting at $0$.
\end{itemize}
\end{itemize}

The LOCA model was trained using an ADAM optimizer, that minimized at each epoch one of the two loss \eqref{eq:white_loss},\eqref{eq:recon_loss}. It was trained using $90\%$ of the given clouds, while the rest was defined as a validation set. The early stopping mechanism was implemented by evaluating the sum of the two losses every $100$ epochs. It saves the minimal value and the weights of the model that achieved this loss. The neural net terminates its training when the minimal loss was not changed in the last $2000$ epochs and loads the saved weights. By following this description we trained LOCA with the sequence of learning rates $10^{-3}$, then we fine tuned it with learning $3\cdot 10^{-4}$, and finished by training it with the learning rate $10^{-4}$.

\begin{figure*}
\centering
\subfigure[ ]{\label{fig:x_mush_appendix} \includegraphics[width=0.4\textwidth] {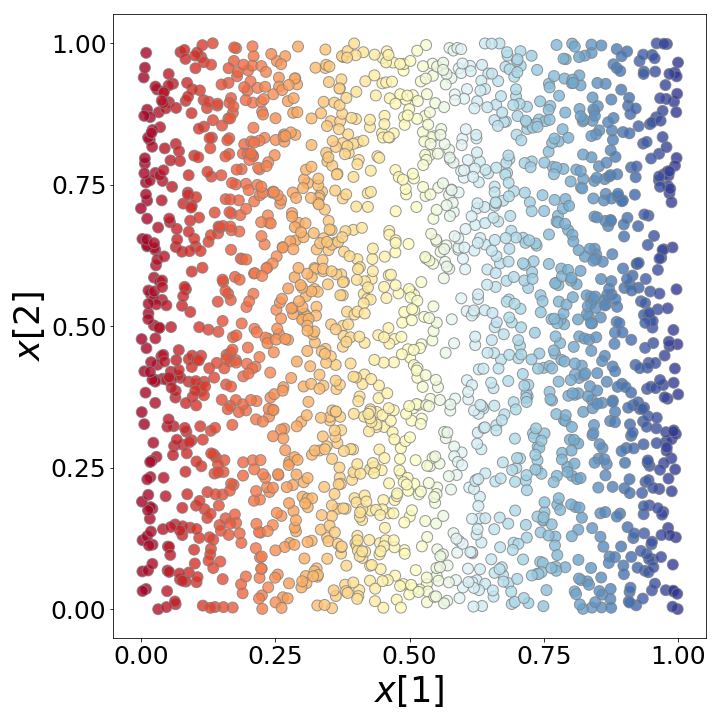}}
\subfigure[ ]{\label{fig:dm_mush} \includegraphics[width=0.4\textwidth] {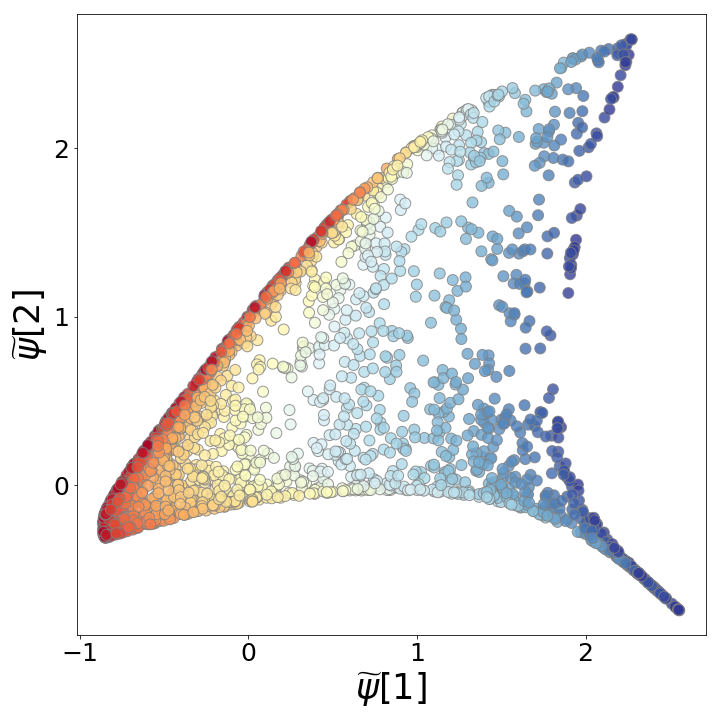}}
\subfigure[ ]{\label{fig:dm_mush_mahal} \includegraphics[width=0.4\textwidth] {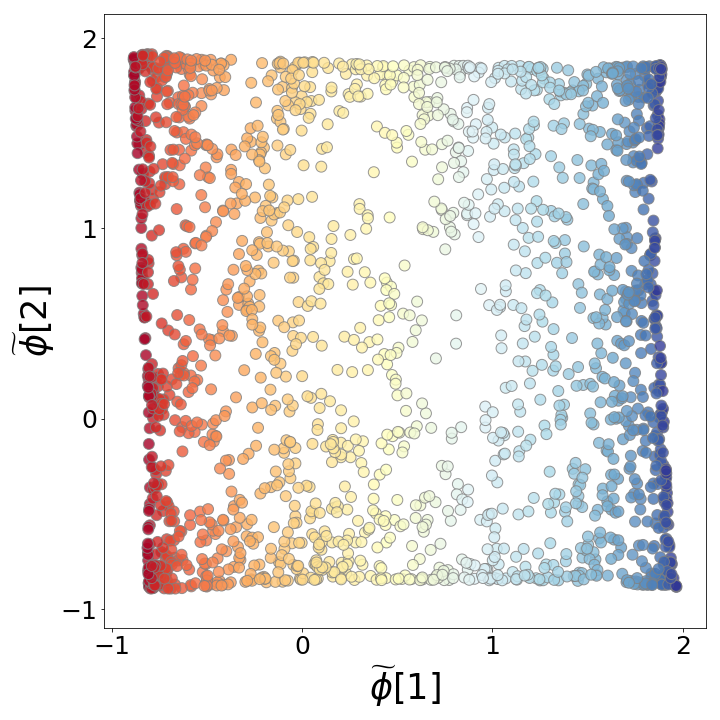}}
\subfigure[ ]{\label{fig:loca_mush} \includegraphics[width=0.4\textwidth] {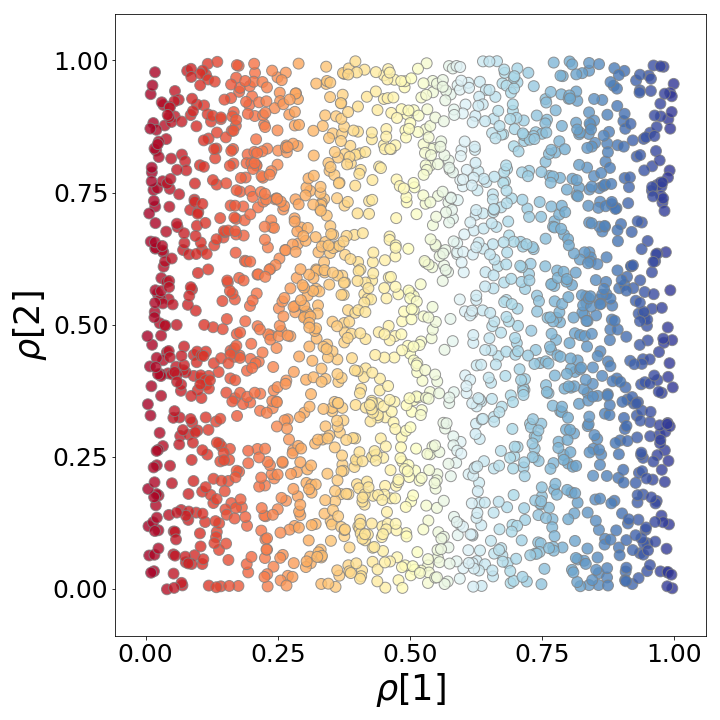}}
\newpage
\caption{Evaluating the isometric property of the proposed embedding (as described in Section \ref{sec:experimets_isometry}). \subref{fig:x_mush_appendix}- The latent representation of the data. 
\subref{fig:dm_mush}- The scaled calibrated embedding of DM. \subref{fig:dm_mush_mahal} The scaled calibrated embedding of A-DM. \subref{fig:loca_mush} The calibrated embedding of LOCA.
The color in both figures correspond to the values of $x[1]$ of the data.}
\end{figure*}
\newpage

\begin{figure*}
\centering
\subfigure[ ]{ \label{fig:spherical_x_appendix}
        \includegraphics[width=0.4\textwidth] {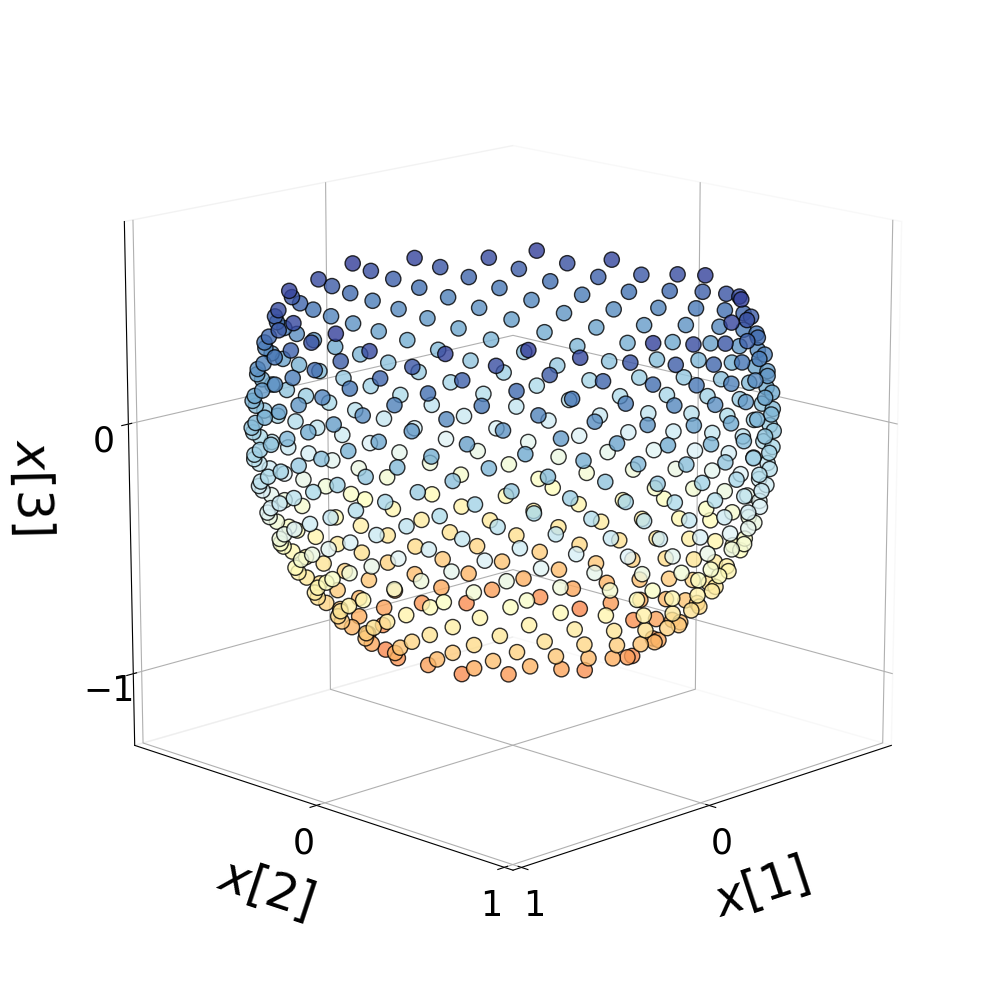}}
        \subfigure[ ]{ \label{fig:spherical_dm}
        \includegraphics[width=0.4\textwidth] {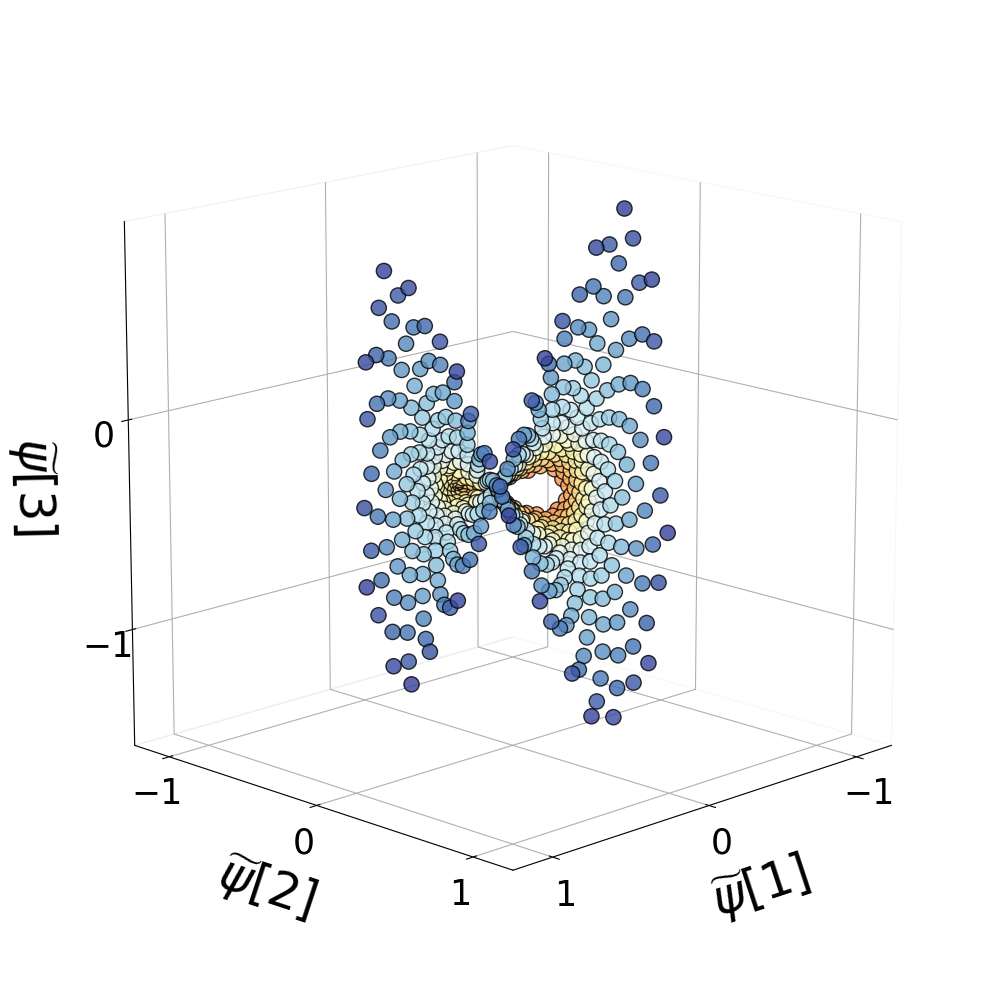}}
        \subfigure[ ]{ \label{fig:spherical_adm}
        \includegraphics[width=0.4\textwidth] {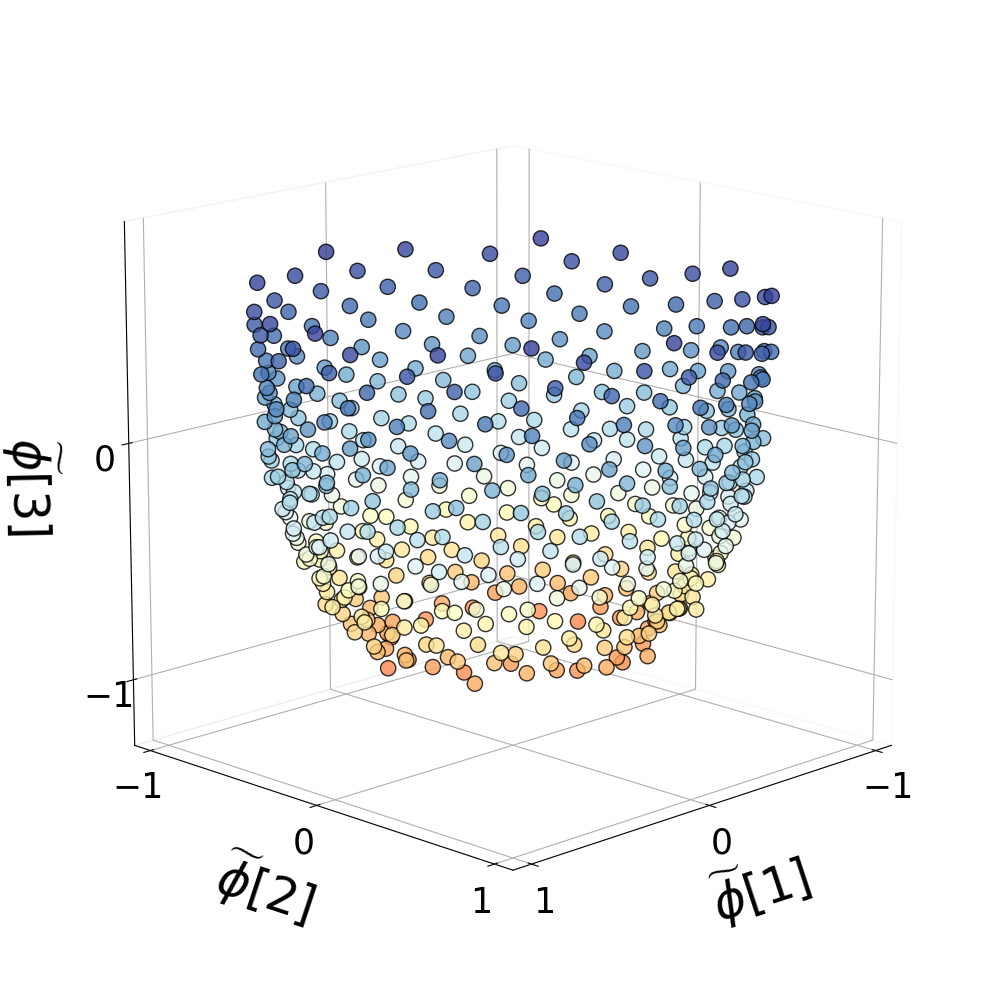}}
        \subfigure[ ]{ \label{fig:spherical_loca} 
        \includegraphics[width=0.4\textwidth] {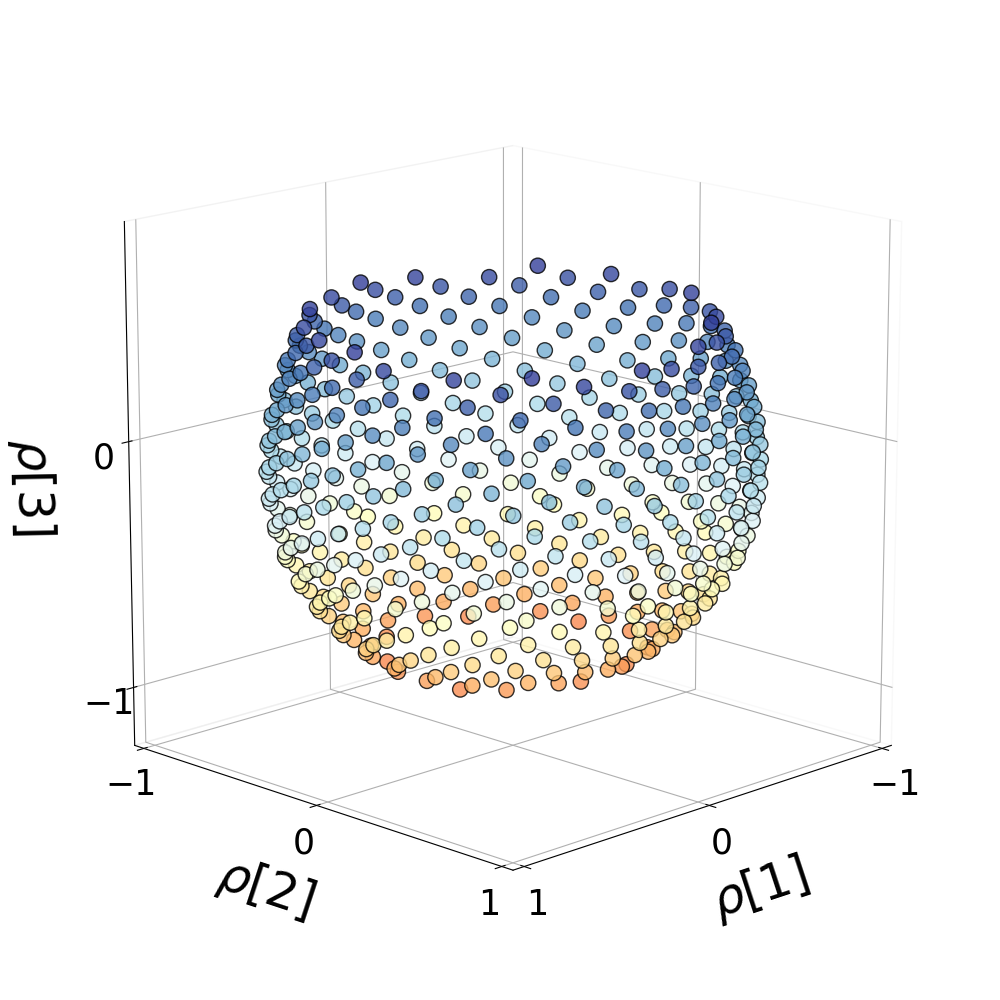}}
    \caption{Evaluating three dimensional embedding of the stereographic projection (described in \eqref{eq:stereographical_x}). We compare the latent representation of the data \subref{fig:spherical_x_appendix} with its embedding based of DM \subref{fig:spherical_dm}, A-DM \subref{fig:spherical_adm} and LOCA \subref{fig:spherical_loca}. 
    The colors used in \subref{fig:spherical_dm},\subref{fig:spherical_adm} and  \subref{fig:spherical_loca} is $\alpha$ defined in \eqref{eq:stereographical_x}}
    \label{fig:stereographical_embeddings}
\end{figure*}

\begin{figure*}
        \subfigure[ ]{ \includegraphics[width=.5\textwidth] {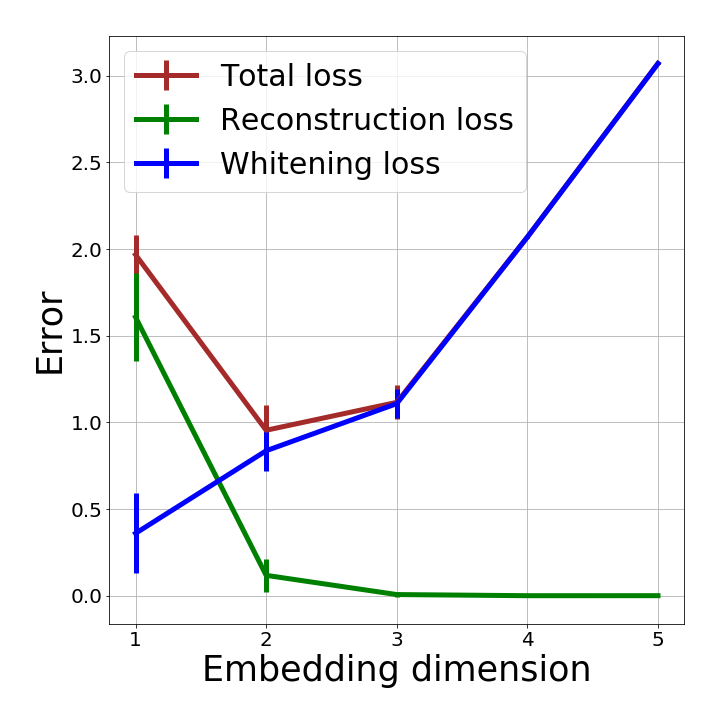}
        \label{fig:stereograhphical_differentDims_val}}
         \subfigure[ ]{ \includegraphics[width=.5\textwidth] {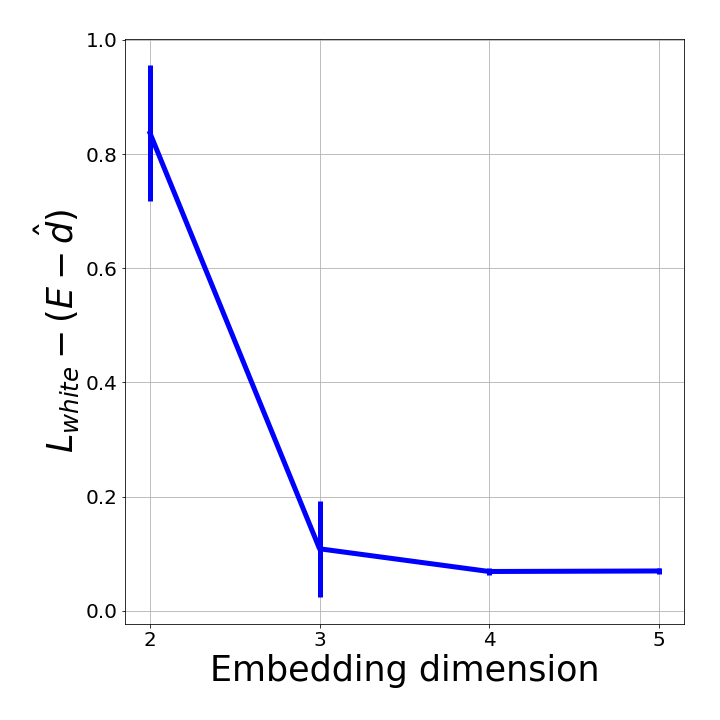}
        \label{fig:stereograhphical_differentDims_whiteAlgo}}
        \caption{A demonstration of the minimal embedding dimension estimation ($E$), based on the procedure described in Section \ref{sec:Loca_curved_manifold}. 
        \ref{fig:stereograhphical_differentDims_val}: Loss values on the validation set for embedding dimensions in the range $1,\ldots,5$. 
        \ref{fig:stereograhphical_differentDims_whiteAlgo}: Plot of suggested quantity for estimating the minimal embedding dimension. This plot suggests that $3$ coordinates are sufficient to represent the data using LOCA.
        }
        \label{fig:algo2_sterographical}
\end{figure*}
\end{appendices}

\end{document}